%% file: main.tex
\documentclass[10pt]{article} %
\usepackage[preprint]{tmlr}

\usepackage[utf8]{inputenc} %
\usepackage[T1]{fontenc}    %
\usepackage[colorlinks]{hyperref}       %
\usepackage{url}            %
\usepackage{booktabs}       %
\usepackage{amsfonts}       %
\usepackage{nicefrac}       %
\usepackage{microtype}      %
\usepackage{xcolor}         %
\usepackage{tikz}
\usepackage{siunitx}

\usepackage{placeins}
\usepackage{graphicx}
\usepackage{subcaption}
\usepackage{adjustbox}
\usepackage{wrapfig}

\usepackage{amsmath}
\usepackage{amssymb}
\usepackage{mathtools}
\usepackage{amsthm}
\usepackage{cancel}
\usepackage[capitalize,noabbrev]{cleveref}
\usepackage{letltxmacro}
\usepackage{algorithm}
\usepackage{algorithmic}
\usepackage{enumitem}

\input{preamble/preamble_shortcuts}
\input{preamble/preamble_math}

\input{preamble/preamble_tikz}

\usepackage{bbm}

\usetikzlibrary{shapes.geometric, arrows, fit, automata, positioning, calc, arrows.meta}
\usetikzlibrary{backgrounds}
\usetikzlibrary{positioning,arrows.meta,shapes.geometric,backgrounds,fit,calc}

\usepackage{dsfont}

\usepackage[dvipsnames]{xcolor}
\hypersetup{citecolor=MidnightBlue}
\hypersetup{linkcolor=MidnightBlue}
\hypersetup{urlcolor=MidnightBlue}

\newcommand{\St}{\mathcal{S}}
\newcommand{\Ac}{\mathcal{A}}

\DeclareRobustCommand{\kl}{\mathrm{kl}}
\DeclareRobustCommand{\invklup}{\mathrm{kl}^{-1,+}}
\DeclareRobustCommand{\invkllow}{\mathrm{kl}^{-1,-}}
\usepackage{nicefrac}

\usepackage[
    natbib=true,
    backend=biber,
    style=apa,
    sorting=ynt,
    sortcites=true,
]{biblatex}                           %

\usepackage{multicol}

\newtheorem{theorem}{Theorem}[section]

\newtheorem{lemma}{Lemma}[section]

\title{Deep Actor-Critics with Tight Risk Certificates}

\author{\name Bahareh Tasdighi \email tasdighi@imada.sdu.dk \\
      \addr Department of Mathematics and Computer Science\\
      University of Southern Denmark
      \tmlrAND
      \name Manuel Haußmann \email haussmann@imada.sdu.dk \\
      \addr Department of Mathematics and Computer Science\\
      University of Southern Denmark
      \tmlrAND
      \name Yi-Shan Wu \email yswu@imada.sdu.dk \\
      \addr Department of Mathematics and Computer Science\\
      University of Southern Denmark
      \tmlrAND
      \name Andres R.\ Masegosa \email arma@cs.aau.dk \\
      \addr Department of Computer Science\\
      Aalborg University
      \tmlrAND
      \name Melih Kandemir \email kandemir@imada.sdu.dk \\
      \addr Department of Mathematics and Computer Science\\
      University of Southern Denmark
     }

\begin{document}

\maketitle

\input{paper.tex}

\printbibliography

\input{appendix.tex}

\end{document}

%% file: preamble/preamble_math.tex
\DeclareRobustCommand{\KL}[2]{\ensuremath{\mathrm{KL}\left(#1\;\|\;#2\right)}}
\DeclareRobustCommand{\klpq}[2]{\ensuremath{\mathrm{kl}\left(#1\;\|\;#2\right)}}

\DeclareRobustCommand{\Ep}[2]{\ensuremath{\mathds{E}_{#1}\left[#2\right]}}
\DeclareRobustCommand{\E}[1]{\ensuremath{\mathds{E}\left[#1\right]}}

\newcommand{\veps}{\varepsilon}

\newcommand{\deq}{\triangleq}  %

\DeclareMathOperator*{\argmax}{arg\,max}

\newcommand{\mcA}{\mathcal{A}}
\newcommand{\mcB}{\mathcal{B}}

\newcommand{\mcD}{\mathcal{D}}
\newcommand{\mcE}{\mathcal{E}}

\newcommand{\mcH}{\mathcal{H}}

\newcommand{\mcP}{\mathcal{P}}

\newcommand{\mcS}{\mathcal{S}}

\newcommand{\mcX}{\mathcal{X}}
\newcommand{\mcY}{\mathcal{Y}}

\newcommand{\mdE}{\mathds{E}}

\newcommand{\mdP}{\mathds{P}}

\newcommand{\mdR}{\mathds{R}}

%% file: preamble/preamble_tikz.tex
\usepackage{tikz}

\makeatletter
\input{preamble/tikzlibrarybayesnet.code.tex}

\makeatother
\usetikzlibrary{backgrounds}
\usetikzlibrary{calc,quotes,arrows.meta}
\usetikzlibrary{shapes.misc,positioning,calc,graphs,arrows}

\pgfdeclarelayer{edgelayer}
\pgfdeclarelayer{nodelayer}
\pgfsetlayers{edgelayer,nodelayer,main}

\definecolor{hexcolor0xbfbfbf}{rgb}{0.749,0.749,0.749}

\tikzset{>=latex}
\tikzstyle{none}   = [inner sep=0pt]
\tikzstyle{line}   = [ -, thick, shorten <=1pt, shorten >=1pt ]
\tikzstyle{arrow}  = [ ->, thick, shorten <=1pt, shorten >=1pt ]
\tikzstyle{ardash} = [ dashed, ->, thick, shorten <=1pt, shorten >=1pt ]

\tikzstyle{box} = [rectangle, minimum width=1.5cm, minimum height=1.5cm,text centered, draw=black, inner sep=7pt]
\tikzstyle{neuron} = [circle, minimum width=4mm, very thick, draw=blue!80!black]

\tikzstyle{empty}=[circle,opacity=0.0,text opacity=1.0,inner sep=0pt]
\tikzstyle{box}=[rectangle,fill=White,draw=Black]
\tikzstyle{filled}=[circle,thick,fill=hexcolor0xbfbfbf,draw=Black]
\tikzstyle{hollow}=[circle,thick,fill=White,draw=Black]
\tikzstyle{param}=[rectangle,fill=Black,draw=Black,inner sep=0pt,minimum width=4pt,minimum height=4pt]
\tikzstyle{paramhollow}=[rectangle,thick,fill=White,draw=Black,inner sep=0pt,minimum
width=4pt,minimum height=4pt]

\usepackage{pgfplots}                               %
\pgfplotsset{compat=newest}
\pgfplotsset{plot coordinates/math parser=false}
\usepgfplotslibrary{statistics}
\newlength\figureheight
\newlength\figurewidth
\newlength\figureheightsmall
\newlength\figurewidthsmall
\definecolor{POSTcolor}{rgb}{0.48, 0.20, 0.58} %
\definecolor{Qcolor}{rgb}{0.00, 0.53, 0.22} %

\makeatletter
\tikzset{
  prefix after node/.style={
    prefix after command={\pgfextra{#1}}
  },
  /semifill/ang/.store in=\semi@ang,
  /semifill/ang=0,
  semifill/.style={
    circle, draw,
    prefix after node={
      \typeout{aaa \semi@ang}
      \let\nodename\tikz@last@fig@name
      \fill[/semifill/.cd, /semifill/.search also={/tikz}, #1]
        let \p1 = (\nodename.north), \p2 = (\nodename.center) in
        let \n1 = {\y1 - \y2} in
        (\nodename.\semi@ang) arc [radius=\n1, start angle=\semi@ang, delta angle=180];
    },
  }
}
\makeatother

%% file: preamble/tikzlibrarybayesnet.code.tex
\usetikzlibrary{shapes}
\usetikzlibrary{fit}
\usetikzlibrary{chains}
\usetikzlibrary{arrows}

\tikzstyle{latent} = [circle,fill=white,draw=black,inner sep=1pt,
minimum size=20pt, font=\fontsize{10}{10}\selectfont, node distance=1]
\tikzstyle{obs} = [latent,fill=gray!25]
\tikzstyle{const} = [rectangle, inner sep=0pt, node distance=1]
\tikzstyle{factor} = [rectangle, fill=black,minimum size=5pt, inner
sep=0pt, node distance=0.4]
\tikzstyle{det} = [latent, diamond]

\tikzstyle{plate} = [draw, rectangle, rounded corners, fit=#1]
\tikzstyle{wrap} = [inner sep=0pt, fit=#1]
\tikzstyle{gate} = [draw, rectangle, dashed, fit=#1]

\tikzstyle{caption} = [font=\footnotesize, node distance=0] %
\tikzstyle{plate caption} = [caption, node distance=0, inner sep=0pt,
below left=5pt and 0pt of #1.south east] %
\tikzstyle{factor caption} = [caption] %
\tikzstyle{every label} += [caption] %

%% file: paper.tex
\begin{abstract}
Deep actor-critic algorithms have reached a level where they influence everyday life. They are a driving force behind continual improvement of large language models through user feedback. However, their deployment in physical systems is not yet widely adopted, mainly because no validation scheme fully quantifies their risk of malfunction. 
We demonstrate that it is possible to develop tight risk certificates for deep actor-critic algorithms that predict generalization performance from validation-time observations. 
Our key insight centers on the effectiveness of minimal evaluation data. A small feasible set of evaluation roll-outs collected from a pretrained policy suffices to produce accurate risk certificates when combined with a simple adaptation of PAC-Bayes theory.
Specifically, we adopt a recently introduced recursive PAC-Bayes approach, which splits validation data into portions and recursively builds PAC-Bayes bounds on the excess loss of each portion's predictor, using the predictor from the previous portion as a data-informed prior. 
Our empirical results across multiple locomotion tasks, actor-critic methods, and policy expertise levels demonstrate risk certificates tight enough to be considered for practical use. 
\end{abstract}

\section{Introduction}\label{sec:intro}

Reinforcement learning (RL) is transforming emerging AI technologies. Large language models incorporate human feedback via RL, thereby continually improving accuracy \citep{Christiano2017,Ziegler2019,deepseekai2025deepseekv3technicalreport}. Generative AI is increasingly integrated into agentic workflows to automate complex decision-making tasks. RL has shown great promise in controlling physical robotic systems. Recent deep actor-critic algorithms learned to make a legged robot walk after only 20 minutes of outdoor training in an online mode \citep{Kostrikov2023Demonstrating}. Model-based extensions of actor-critic pipelines can also achieve sample-efficient visual-control tasks in diverse settings \citep{Hafner_2025,zhang2023storm}. Despite promising results observed in experimental conditions, RL is used far less than classical approaches in physical robot control. This opportunity has largely been missed because deep RL algorithms are overly sensitive to initial conditions and can change behavior drastically during training. Embodied intelligent systems pose a high risk of causing harm when generalization performance differs significantly from observed validation performance. Predictable generalization performance is even more critical when these systems update their behavior based on interactions with humans.

There has been an effort to use learning-theoretic approaches to train high-capacity predictors with risk certificates, i.e., bounds that guarantee a predictor’s generalization performance. Typically, this performance is estimated from observed validation results, which may be misleading. \emph{Probably Approximately Correct Bayesian (PAC-Bayes) theory} \citep{mcallester1999pac,alquier2024user} provides risk certificates for stochastic predictors relative to a prior distribution over the hypothesis space. In this framework, the computationally prohibitive capacity term is reduced to a Kullback-Leibler divergence between the posterior and the prior, enabling the incorporation of domain knowledge into the analysis. Since stochastic policies are often considered, relying on PAC-Bayes is a natural choice.

\begin{figure*}
    \centering
    \adjustbox{max width=0.98\textwidth}{
    \begin{tikzpicture}[
        general node/.style={draw=purple!40!black,fill=purple!10!white, thick,rounded corners,align=center, inner sep=3pt, font=\footnotesize\sffamily, minimum height=4em},
        arrow/.style={-{Stealth[scale=1.2]}, thick, draw=gray!90},
        node distance=2.5em
    ]
    \node[general node] (ac) {Agent training\\ \tiny{e.g., REDQ, SAC, PPO}};
    \node[general node, right=3.5em of ac] (policy) {Data collection\\\tiny{interact with the environment}\\[-0.5em]\tiny{given the policy}};
    
    \node[general node, right=3.5em of policy] (evaluation) {Distribution fitting\\\tiny{fit posteriors}\\[-0.5em]\tiny{on data splits}};
    
    \node[general node, right=4.5em of evaluation] (computation) {Risk certificate creation\\\tiny{compute the recursive}\\[-0.5em]\tiny{PAC-Bayes bound}};

    \draw[arrow] (ac) edge[font=\tiny,align=center,above]["policy"] (policy);
    
    \draw[arrow] (policy) edge[align=center,font=\tiny]["data"] (evaluation);
    
    \draw[arrow] (evaluation) edge[align=center,font=\tiny] ["parameters"] (computation);
    \end{tikzpicture}}
    \caption{The four-step procedure to generate tight risk certificates for deep actor-critic algorithms.}
    \label{fig:workflow}
\end{figure*}

PAC-Bayes is the first and remains the most promising method for providing meaningful risk certificates for deep neural networks \citep{dziugaite2017computing,perez2021tighter,lotfi22pacbayes}.
Further studies have improved the tightness (i.e., precision) of these certificates through the following techniques: (i) pretraining probabilistic neural networks on held-out data and using them as \emph{data-informed priors} \citep{ambroladze2006tighter,dziugaite21ontherolw}; (ii) using pretrained networks as first-step predictors and developing PAC-Bayes guarantees on the residual of their predictions, termed the \emph{excess loss}; and (iii) recursively repeating the first two steps on multiple data splits, a recent method known as \emph{Recursive PAC-Bayes} \citep{wu2024recursive}.
The scope of these developments has thus far been limited to simple classification tasks with feedforward neural networks. Their application to deep actor–critic algorithms remains open, primarily because mainstream PAC-Bayes bounds assume i.i.d.\ datasets, whereas RL assumes a controlled Markov chain. 

We present a simple recipe for providing risk certificates for deep, model-free actor–critic architectures. We find that, contrary to expectations, the three modern PAC-Bayesian learning techniques mentioned above can successfully handle the high variance of Monte Carlo samples collected by running a pretrained policy network for multiple episodes in evaluation mode. Our approach proposes self-certified training of probabilistic neural networks on different splits of an i.i.d.\ dataset containing return realizations of the policy, computed by first-visit Monte Carlo and post-processed through a simple thinning approach. 
We recursively build a PAC-Bayes bound on the excess losses of these networks, deriving a new adaptation of \citet{wu2024recursive}. \cref{fig:workflow} illustrates our risk-certificate generation workflow.  
We evaluate the approach for several classical actor–critic agents, PPO~\citep{schulman2017proximal}, SAC~\citep{haarnoja2018soft}, and REDQ~\citep{chen2021randomized}, across three benchmark suites.
Our results indicate that the risk certificates become tighter as the recursion depth increases. The final bounds are sufficiently tight for practical use. Furthermore, the tightness of the risk certificates is proportional to the policy’s expertise.

\section{Background}\label{sec:background}

\subsection{The State of the Art of Model-free Deep Actor-Critic Learning}

Consider a set of states $\St$ an agent may be in and an action space $\Ac$ from which the agent can choose actions to interact with its environment. Denote by $\Delta(\St)$ and $\Delta(\Ac)$ the sets of probability distributions defined on $\St$ and $\Ac$, respectively. We define a Markov Decision Process (MDP) \citep{puterman2014markov} as the tuple ${M=\langle\St, \Ac, r, P, P_0, \gamma\rangle}$, where $r: \St \times \Ac \rightarrow [0,R]$ is a bounded reward function, ${P: \St \times \Ac  \times \St \rightarrow[0,1]}$ is the state-transition kernel conditioned on a state-action pair;  specifically $P(s'|s,a)$ is the probability distribution of the next state $s'\in\St$ given the current state-action pair $(s,a) \in \St \times \Ac$. We denote the initial-state distribution by $P_0 \in \Delta(\St)$, the discount factor by $\gamma \in (0,1)$, and let ${\pi: \mcS\times\mcA \to [0,1]}$ be a policy. The goal of RL is to learn a policy that maximizes the expected discounted return,  $\pi_* := \argmax_{\pi \in \Pi} \Ep{\tau_{\pi}}{\sum_{t=0}^\infty \gamma^t r(s_t,a_t) }$. The expectation is taken with respect to the trajectory ${\tau_{\pi} := (s_0, a_0, s_1, a_1, s_2, a_2, \ldots )}$ of states and actions generated when a policy $\pi$ chosen from a feasible set $\Pi$ is executed. We refer to $\pi_*$ as the optimal policy. The exact Bellman operator for a policy $\pi$  is defined as 
\begin{equation}
    T_\pi Q(s,a) := r(s,a) + \gamma\Ep{s' \sim P(\cdot|s,a)}{Q(s',\pi(s'))} \label{eq:bellman_operator}
\end{equation}
for some function $Q: \St \times \Ac \rightarrow \mdR$. 
The unique fixed point of this operator is the true action-value function $Q_\pi$, which maps a state-action pair $(s,a)$ to the expected discounted sum of rewards the policy $\pi$ collects when executed from $(s,a)$. In other words, the equality ${T_\pi Q(s,a) = Q(s,a)}$ holds if and only if ${Q(s,a)=Q_\pi(s,a), \forall (s,a)}$. Any other $Q$ incurs an error $(T_\pi Q(s,a) - Q(s,a))^2$, called the {\it Bellman error}. Common deep actor-critic methods approximate the true action-value function $Q_\pi$ by one-step Temporal Difference (TD) learning that minimizes ${L(Q,\pi) := \mathbb{E}_{s \sim P_\pi} [ (T_\pi Q(s,a) - Q(s,a))^2]}$ with respect to $Q$, given a data set $\mathcal{D}$ and
${P_\pi(s' \in A)\! =\! \mathbb{E}_{s \sim P_0} \left[\sum_{t > 0} P(s_t \in A | s_0 =s, \pi(s))\right ]}$
which is defined as the state-visitation distribution of policy $\pi$ for some event $A$ that belongs to the $\sigma$-algebra of the transition probability distribution. Because the transition probabilities are unknown, the expectation term in \eqref{eq:bellman_operator} cannot be computed. Instead, the observed transitions are used to approximate it with a  single-sample Monte Carlo estimate, yielding the training objective below:
\begin{align*}
    \widetilde{L}(Q) := 
    \mdE_{s\sim P_\pi}\Big[
    \Ep{s'\sim P(\cdot |s,\pi(s))}{(r(s,a) + \gamma Q(s',\pi(s')) -Q(s,a))^2}\!\Big].
\end{align*}
A deep actor-critic algorithm fits a neural-network function approximator $Q$, referred to as the critic, to a set of observed tuples ${(s,a,s')}$ stored in a replay buffer $\mathcal{D}$ by minimizing an empirical estimate of the stochastic loss: $\widehat{L}_{\mathcal{D}}(Q) := 1/|\mathcal{D}| \sum_{(s,a,s') \in \mathcal{D}} (\widetilde{T}_\pi Q(s,a,s') -Q(s,a))^2$. The critic is subsequently used to train a policy network, or actor, $\pi' \gets \arg \max_\pi \mathbb{E}_{s \sim P_\pi} [Q(s,\pi(s))]$. It is common practice to adopt the \emph{Maximum-Entropy Reinforcement Learning}  approach \citep{haarnoja2018soft, haarnoja2018soft-app} to balance exploration and exploitation, ensuring effective training. The approach supplements the reward function $r(s,a)$ with a policy-entropy term $\mathbb{H}[\pi(\cdot|s)]$, scaled by a hyperparameter $\alpha \geq 0$, tuned jointly with the actor and critic, i.e., ${r_\text{MaxEnt}\deq r(s,a) + \alpha \mathbb{H}[\pi(\cdot|s)]}$. 

Performing off-policy TD learning with deep neural nets is notoriously unstable which is often attributed to the \emph{deadly triad} \citep{sutton2018reinforcement}. The main source of instability is the accumulation of errors from approximating $T_\pi Q$ by its Monte Carlo estimate. Strategies to improve stability include maintaining Polyak-updated target networks \citep{lillicrap2015continuous} and learning twin critics while using the minimum of their target-network outputs in Bellman target calculation \citep{fujimoto2018addressing}. Empirically, training an ensemble of critic networks in a maximum-entropy setup largely mitigates these stability issues. 
We adopt PPO~\citep{schulman2017proximal}, SAC~\citep{haarnoja2018soft}, and REDQ~\citep{chen2021randomized}, three popular actor-critic methods for model-free continuous control, as our representative approaches.
This choice is pragmatic rather than restrictive allowing us to trade the computational cost of a broader exploration of algorithms for a deeper, more comprehensive empirical evaluation of a single one. 

\subsection{Developing Risk Certificates with PAC-Bayes Bounds}
PAC-Bayes \citep{mcallester1999pac,alquier2024user} offers a powerful framework for understanding and controlling the generalization performance of learning algorithms by integrating prior beliefs with information obtained from data.
\emph{PAC-Bayesian learning} employs modern machine learning techniques to model $\rho$ with complex function approximators and fit them to data. 
It has been successfully applied to both image classification \citep{dziugaite2017computing, wu2024recursive} and regression tasks \citep{reeb2018learning}. 
Its application to reinforcement learning has thus far been limited to the design of critic training losses, without rigorous quantification of the tightness of the performance guarantees \citep{tasdighi2023pac,tasdighi2024deep}.

\paragraph{Notation.}\label{par:notation}
Let $\mathcal{H}: \mathcal{X} \rightarrow \mathcal{Y}$ be a set of feasible hypotheses, and let ${\ell: \mathcal{Y} \times \mathcal{Y} \rightarrow [0,1]}$ be a bounded loss function.\footnote{Our discussion generalizes directly to any bounded loss within an interval $[a,b]$ with $a,b \in \mdR$.} 
Further, let ${L(h) = \Ep{(x,y) \sim P_D}{\ell(h(x), y)}}$ denote the expected error, where $P_D$ is a distribution on $\mcX\times \mcY$. The empirical loss is ${\hat{L}(h) = \frac{1}{N} \sum_{i=1}^N \ell(h(x_i), y_i)}$ for a data set ${\mcD = \{(x_n,y_n):n\in\{1,\ldots, N\}\}}$ of size $N$ with $(x_n,y_n)\sim P_D$. $\mcP$ denotes the set of distributions on $\mcH$.
For two distributions $\rho, \rho_0$ on $\mcH$, the Kullback–Leibler (KL) divergence is defined as ${\KL{\rho}{\rho_0} \deq \Ep{h\sim \rho}{\log \rho(h) - \log \rho_0(h)}}$. We use ${\klpq{p}{q} \deq p\log (p/q) + (1 - p)\log ((1 - p)/(1-q))}$ to denote the KL divergence between two Bernoulli distributions.
We define the upper inverse of $\klpq{\cdot}{\cdot}$ as ${\invklup(\hat p,\veps)\deq \max\{p: p \in [0,1] \mid \klpq{\hat p}{p} \leq \veps\}}$ and the lower one as ${\invkllow(\hat p, \veps) \deq \min\{p: p\in[0,1] \mid \klpq{\hat p}{p}\leq \veps\}}$.
PAC-Bayesian analysis \citep{mcallester1999pac,shawe1997pac} develops bounds on the \emph{expected loss} $\Ep{h\sim \rho}{L(h)}$ under a posterior distribution $\rho$ with respect to a prior distribution $\rho_0$ that hold with high probability. That is, they provide \emph{risk certificates} for the generalization error.
For brevity, we use $\Ep{\rho}{\cdot} = \Ep{h\sim \rho}{\cdot}$ throughout this paper. 
In the context of PAC-Bayes, the terms \emph{posterior} and \emph{prior} refer to distributions dependent on and independent of the validation data, respectively. They are not to be interpreted in a Bayesian sense as being linked by a likelihood.\footnote{See \citet{Germain2016} for results linking PAC-Bayes and Bayesian inference.}
The choice of bounds that yields the tightest risk certificates depends on the specific use case; see, e.g., \citet{alquier2024user} for a recent introduction and a survey of various PAC-Bayesian bounds. 
Here, we rely on bounds derived from the $\kl$ divergence, as they are tighter than alternatives when no additional information about the data distribution is available, while noting that the same arguments apply to any other PAC-Bayesian bound.

\subsubsection{PAC-Bayes-Split-$\kl$ Bound}
\citet{wu2022split} introduce a PAC-Bayes bound for random variables that take values in intervals $[a,b]$ by splitting them into components that individually satisfy the constraints of a $\kl$-inequality 
(see \cref{lem:split-kl}).

Define ${\tilde \ell: \mcY\times \mcY \to [a,b]}$, where $a,b \in \mdR$. For $\mu \in [a,b]$, define $\tilde \ell^+ = \max\{0,\tilde\ell - \mu\}$ and ${\tilde \ell^- = \max\{0,\mu - \tilde \ell\}}$. 
${\tilde L^+(h) = \mdE_{(x,y)\sim P_D}[\tilde \ell^+(h(x),y)]}$ and ${\hat{\tilde L}^+(h) = \frac1N\sum_{n=1}^N\tilde \ell^+(h(x_n),y_n)}$ are the expected and empirical losses. $L^-$ and $\hat{\tilde L}^-$ are defined analogously. We cite the following bound.
\begin{theorem}[PAC-Bayes-Split-$\kl$ inequality \citep{wu2022split}]\label{thm:pacbayes-split-kl}
Let $\tilde \ell$ and the remaining loss terms be defined as above. Then, for any $\rho_0$ on $\mcH$ independent of $\mcD$, any $\mu\in [a,b]$, and any $\delta \in (0,1)$,
\begin{align*}
    &\exists \rho \in \mcP: \mdE_{\rho}[\tilde L(h)] \geq \mu \\
    &\vspace{1em}+\!(b \!-\! \mu)\invklup\biggl(\!\frac{\mdE_{\rho}[\hat{\tilde L}^+(h)]}{b-\mu},\!\frac{\mathrm{KL}(\rho\!\mid\mid\!\rho_0)\!+\!\ln(4\sqrt{N}/\delta)}{N}\!\biggr)\\
    &\vspace{1em}-\!(\mu \!-\! a)\invkllow\biggl(\!\frac{\mdE_{\rho}[\hat{\tilde L}^-(h)]}{\mu - a},\!\frac{\mathrm{KL}(\rho\!\mid\mid\!\rho_0)\!+\!\ln(4\sqrt{N}/\delta)}{N}\!\biggr),
\end{align*}
with probability at most $\delta$.
\end{theorem}
\begin{proof}
    The theorem follows by applying 
    \cref{lem:split-kl} 
    to the decomposition ${\mdE_{\rho}[\tilde L(h)] = \mu + \mdE_{\rho}[\tilde L^+(h)] - \mdE_{\rho}[\tilde L^-(h)]}$.
\end{proof}

\subsubsection{Recursive PAC-Bayes Bound}\label{sec:rpb}

\paragraph{Data-informed Prior.}
The tightness of PAC-Bayesian bounds is governed by the KL divergence between the posterior $\rho$ and the prior $\rho_0$. The better the prior guess, the tighter the bound.
Because the prior must be independent of the observed data, a common choice is to select a prior as uniform as possible over the hypothesis space.
To improve upon this naive choice, \citet{ambroladze2006tighter} proposed splitting the observed data into two disjoint subsets, $S_0$ and $S_1$, i.e., $\mcD= S_0 \cup S_1$, using $S_0$ to infer a \emph{data-informed prior} and $S_1$ to subsequently evaluate the bound. This balances the benefit of a better prior with the cost of having fewer observations to evaluate the bound.

\paragraph{Excess Loss.}
The \emph{excess loss} $L^\text{exc}(h)$ with respect to a reference hypothesis $h^*\in \mcH$ is defined as $L^\text{exc}(h) = L(h) - L(h^*)$.
The excess-loss concept allows the expected loss to be decomposed as
    ${\Ep{\rho}{L(h)} = \Ep{\rho}{L(h) - L(h^*)} + L(h^*)}$.
Using $S_0$ to construct both the prior $\rho_0$ and the reference $h^*$, \citet{mhammedi2019pac} showed that, assuming $L(h^*)$ is close to $L(h)$, the excess loss has lower variance and thus yields a more efficient bound, while a bound on $L(h^*)$ is independent of $\KL{\rho}{\rho_0}$ and can be obtained using standard generalization guarantees.

\paragraph{Recursive PAC-Bayes.}
\citet{wu2024recursive} generalized the excess loss further by introducing a scaling factor $\kappa < 1$ to maintain a diminishing effect of recursions:
${\Ep{\rho}{L(h)} = \Ep{\rho}{L(h) - \kappa \Ep{\rho_0}{L(h^*)}} + \kappa \Ep{\rho_0}{L(h^*)}}$. Here, the first term reflects the excess loss with respect to a scaled version of the expected reference hypothesis loss under the prior $\rho_0$. The second term, in turn, is an expected loss similar to the one on the left-hand side of the equation. Instead of adhering to a binary split $\mcD = S_0\cup S_1$ such that $S_0\cap S_1=\emptyset$, they proposed extending this decomposition recursively by partitioning $\mcD$ into $T$ disjoint subsets, $\mcD = \bigcup_{t=1}^T S_t$, and they define $S_{\leq t} = \bigcup_{s=1}^t S_s$ and $S_{\geq t}=\bigcup_{s=t}^T S_s$.
Their recursion is given by 
\begin{align}\label{eq:rpb}
\begin{split}
\Ep{\rho_t}{L(h)} &= \Ep{\rho_t}{L(h) - \kappa_t\Ep{\rho_{t-1}}{L(h)}} + \kappa_t\Ep{\rho_{t-1}}{L(h)},
\end{split}
\end{align}
for $t \geq 2$, and $\kappa_1,\ldots,\kappa_T$ are scaling factors.
The distributions $\rho_1,\ldots, \rho_{T} \in \mcH$ form a sequence such that $\rho_t$ depends solely on $S_{\leq t}$. %

While \citet{wu2024recursive} formulated their final recursive bound directly for a zero-one loss and PAC-Bayes split-kl bounds \citep{wu2022split}, 
we present their result first in a general, loss-agnostic form before constructing a specific bound in the next section.

\begin{theorem}{(\emph{Recursive PAC-Bayes bound.})}\label{thm:rpb}
Let $\mcD = S_1\cup \cdots \cup S_T$ be a disjoint decomposition of the set of observations $\mcD$. Let $S_{\leq t}$ and $S_{\geq t}$ be as defined above, $N = |\mcD|$, and $N_t = |S_{\geq t}|$.
Let $\kappa_1,\ldots, \kappa_T$ be a sequence of scaling factors, where $\kappa_t$ is allowed to depend on $S_{\leq t-1}$.
Let $\mcP_t$ be the set of distributions on $\mcH$ that are allowed to depend on $S_{\leq t}$, and $\rho_t \in \mcP_t$.
Then, for any $\delta \in (0,1)$,
\begin{equation*}
    \mdP\left(\exists t \in [T], \rho_t \in \mcP_t\!\text{ such that }\!\Ep{\rho_t}{L(h)} \geq \mcB_t(\rho_t)\big.\right) \leq \delta,
\end{equation*}
where $\mcB_t(\rho_t)$ is a generic PAC-Bayesian bound on $\Ep{\rho_t}{L(h)}$ defined recursively as follows.
\begin{equation*}
    \mcB_t(\rho_t) = \mcE_t(\rho_t,\kappa_t) + \kappa_t \mcB_{t-1}(\rho^*_{t-1}),
\end{equation*}
where $\mcB_1(\rho_1)$ is a PAC-Bayes bound on $\Ep{\rho_1}{L(h)}$ with an uninformed prior, and $\mcE_t(\rho_t,\kappa_t)$ is a PAC-Bayes bound on the excess loss $\mdE_{\rho_t}\big[L(h) - \kappa_t\Ep{\rho^*_{t-1}}{L(h')}\big]$.
\end{theorem}
\begin{proof}
Because $\mcB_1(\rho_1)$ and $\mcE_t(\rho_t, \kappa_t)$ are PAC-Bayes bounds by assumption, we have 
\begin{align*}
&\mdP\left(\exists \rho_1 \in \mcP_1: \Ep{\rho_1}{L(h)}\geq \mcB_1(\rho_1)\right)\leq \delta/T,\\ 
\text{and} \quad &\mdP\big(\exists \rho_t \in \mcP_t: \mdE_{\rho_t}\big[L(h) - \kappa_t\Ep{\rho_{t-1}^*}{L(h')}\!\big] \geq \mcE_t(\rho_t,\kappa_t)\big) \leq \delta/T \text{ for } t \in \{2,\ldots,T\}.
\end{align*}
The claim follows by expected loss decomposition and the recursion.
\end{proof}

\section{Recursive PAC-Bayesian Risk Certificates for Reinforcement Learning}\label{sec:rpbrl}
The general recursive PAC-Bayes framework introduced in \cref{sec:rpb} provides a generic method for certifying the expected loss of stochastic predictors. To apply this framework to reinforcement learning, hypotheses, data, and losses are interpreted in terms of policies, evaluation roll-outs, and return prediction. 
Obtaining risk certificates involves four steps, following the conceptual structure in \cref{fig:workflow}.

\paragraph{\emph{(i) Agent Training.}}
An actor–critic algorithm, e.g., REDQ \citep{chen2021randomized}, is trained until convergence or until a computational budget is exhausted. The resulting policy $\pi$ is then fixed. The objective is not to optimize $\pi$ further, but to certify its performance. 

\paragraph{\emph{(ii) Data Collection.}}
Execute the frozen policy for $N_\text{ep}$ episodes in evaluation mode (no exploration noise, no parameter updates). 
For each episode, store the trajectory $\tau$ and use it to compute discounted returns $G_t = \sum_{i=t}^\infty \gamma^{i-t} r_i$ for each visited state $s_t$.  
The discounted return is used as the prediction target rather than an undiscounted sum of rewards. 
Short-term risks tend to be more relevant for decisions, as longer-term risks depend on an increasing set of external, often unaccountable, factors. Discounted rewards also serve as a proxy for lifelong learning and policy evaluation, as they generalize to non-episodic data. 
Although the original policy might be trained on discounted returns in step \emph{(i)}, a valid bound can also be constructed by computing undiscounted rewards from data collected in \emph{(ii)}. 
As consecutive samples are highly correlated, a simple thinning strategy, described in the appendix, is applied to better approximate the i.i.d.\ assumption required by PAC-Bayes and to form a data set $\mcD = \{(s_n, G_n)\}_{n=1}^N$.
Although a PAC-Bayesian bound tightens as the number of data points increases, it is observed that even a relatively small number of evaluation roll-outs is sufficient to obtain tight results. 

\paragraph{\emph{(iii) Fitting Posteriors via Recursive Training.}}
Partition the data into $T$ disjoint subsets ${\mcD = S_1\cup \cdots \cup S_T}$. 
Train a sequence of $T$ Bayesian neural networks (BNNs), where the $t$-th network is trained on $S_{\leq t}$ by minimizing the PAC-Bayes bound $\mcE_t(\rho_t,\kappa_t)$ from \cref{thm:rpb} to infer the posterior $\rho_t$ over the BNN parameters. $\rho_{t-1}$ serves as a data-informed prior for $t>1$, and an uninformative prior is used for $\rho_0$. $\kappa_t=0.5$ is used in this work.

\paragraph{\emph{(iv) Risk Certificate Construction.}}
Apply \cref{thm:rpb} to construct recursive bounds as follows. 

\paragraph{A bound for $\mcB_1$.}
As $\hat L(h)$ is bounded in $[0,B]$, its expectation is rescaled, and 
\begin{equation*}
    \mcB_1(\rho_1)\!=\! B\invklup\!\biggl(\!\frac{\mdE_{\rho}[\hat L(h)]}{B},\!\!\frac{\KL{\rho_1\!\!}{\!\!\rho_0^*}\!+ \!\ln(2T\!\sqrt{N}\!/\!\delta)}{N}\!\biggr)\!,
\end{equation*}
is chosen, where $\rho_0^*$ is a data-independent prior distribution on $\mcH$.
Given the result in 
\cref{thm:pac_bayes_kl}, 
this is a PAC-Bayesian bound on $\Ep{\rho_1}{L(h)}$, i.e., 
\begin{equation*}
    \mdP\bigl(\exists \rho_1 \in \mcP_1: \Ep{\rho_1}{L(h)} \geq \mcB_1(\rho_1)\bigr) \leq \delta/T.
\end{equation*}

\paragraph{A bound for $\mcE_t$.}
Let ${L^\text{exc}_t(h) = L(h) - \kappa_t\Ep{\rho_{t-1}}{L(h')} \in [-\kappa_t B,B]}$.
For $\mu \in [-\kappa_t B, B]$, define ${L^{\text{exc}+}_t(h) = \max\{0,L^\text{exc}_t(h) - \mu\}}$ and ${L^{\text{exc}-}_t (h) = \max\{0,\mu - L^\text{exc}_t(h) \}}$, with $\hat L^{\text{exc}+}_t(h)$ and $\hat L^{\text{exc}-}_t(h)$ as their empirical analogues. Set 
\begin{align*}
    &\mcE_t(\rho_t) = \mu \\
    &\vspace{1em}+(B-\mu)\invklup\!\big(\mdE_{\rho_t}[\hat{L}^{\text{exc}+}_t(h)]\big/(B - \mu),\Psi_t \big)\\
    &\vspace{1em}- (\mu + \kappa_t B)\invkllow\big(\mdE_{\rho_t}[\hat{L}^{\text{exc}-}_t(h)]\big/(\mu + \kappa_t B),\Psi_t \big),
\end{align*}
where $\Psi_t = \KL{\rho_t}{\rho_{t-1}^*} + \ln(4T\sqrt{N_t}/\delta)/N_t$, and
$\rho_{t-1}^*$ is a distribution on $\mcH$ informed by $S_{\leq t-1}$. Via \cref{thm:pacbayes-split-kl}, this is a PAC-Bayesian bound on $\Ep{\rho_t}{L_t^\text{exc}}$ that holds with probability at least $1 - \delta/T$, i.e., 
\begin{equation*}
    \mdP\bigl(\exists \rho_t \in \mcP_t \text{ such that } \Ep{\rho_t}{L^\text{exc}_t(h)}\geq \mcE_t(\rho_t)\bigr) \leq \delta/ T.
\end{equation*}

Applying this construction recursively with $T$ steps therefore yields a recursive PAC-Bayesian bound that holds with probability at least $1 - \delta$.

\section{EXPERIMENTS}\label{sec:exp}
The aim of our experiments is to address the following questions.
\emph{Q1}: Can the test-time return of a policy be predicted with high precision across a range of environments and policies with varying levels of expertise?
\emph{Q2}: What is the influence of the structure of a PAC-Bayesian bound?
\emph{Q3}: How does the validation set size affect the tightness of the risk certificate guarantee?

\subsection{Experiment Design}\label{sec:exp:design}
We evaluate our certificate-generation pipeline at an error tolerance of $\delta = 0.025$ on three popular representative actor-critic algorithms:
PPO~\citep{schulman2017proximal}, a classical and still widely used approach due to its robustness and simplicity despite comparatively low sample efficiency;
SAC~\citep{haarnoja2018soft}, a strong off-policy baseline that combines stability with high sample efficiency through entropy-regularized actor-critic learning; and
REDQ~\citep{chen2021randomized}, a more recent method that exemplifies advances in sample efficiency by leveraging randomized ensembles of Q-functions.
Training and hyperparameter details for each method are provided in the appendix.

\paragraph{Environments.} 
We rely on environments from three benchmarks. 
All three algorithms are evaluated in the main paper on three MuJoCo~\citep{todorov2012mujoco} environments: \emph{Half-Cheetah}, \emph{Humanoid}, and \emph{Hopper}, due to their widespread use in the community and the representative value of the platforms for real-world use cases. 
The appendix includes REDQ results on two additional MuJoCo environments and three environments each from DM-Control~\citep{tassa2018deepmind} another popular locomotion benchmark and MetaWorld~\citep{yu2020meta}, which focuses on robotic manipulation.

\paragraph{Data Generation}
We train each agent for up to \num{300000} steps, after which the policy is frozen.
The learned policy is then run in evaluation mode for another 200 episodes, the first half of which are used to fit the bound, and the remainder serve as a test dataset to evaluate generalization performance.
The discounted return on the test set is predicted by fitting a PAC-Bayes bound using the validation set.

\paragraph{Policy Instances.}
We define a policy instance as the output of a single policy-training round. In our experiments, we consider five policy instances, each obtained by training with a different initial seed.
Due to the stochastic nature of initialization and training, each instance follows a unique trajectory.
We construct individual bounds for each instance and account for randomness in the risk-certificate generation process by repeating the procedure five times for every policy instance.
To model policies of varying quality, we evaluate three training stages of each policy, each reflecting a different level of expertise: \emph{Starter}, a policy trained for \num{100000} steps; \emph{Intermediate}, trained for \num{200000} steps; and \emph{Expert}, trained for \num{300000} steps, after which no further performance improvements were observed.

\paragraph{Baselines.} 
We design our baselines with the following goals:
\emph{(i)} how well a PAC-Bayes bound predicts test-time performance,
\emph{(ii)} whether informative priors yield tighter guarantees,
\emph{(iii)} whether the bound tightens when the recursive scheme is used, and
\emph{(iv)} whether increasing recursion depth improves tightness.
As this is the first work to evaluate generalization bounds tailored to continuous control with deep actor-critics, there are no existing baselines for comparison.
We consider two non-recursive baselines: \emph{NonRec-NonInf}, a PAC-Bayes-kl inverse bound 
(see \cref{thm:pac_bayes_kl}) 
with a non-informative prior that is independent of the training data, and \emph{NonRec-Inf}, a data-informed variant in which the dataset is split equally into $\mcD = \mcD_{\text{prior}} \cup \mcD_{\text{bound}}$, allowing the prior to depend on $\mcD_{\text{prior}}$ and the empirical loss to be computed on $\mcD_{\text{bound}}$.
We evaluate our approach at two recursion depths, $T=2$ (\emph{Rec T=2}) and $T=6$ (\emph{Rec T=6}), to assess the effect of recursion.

\paragraph{Performance Metrics.}
We evaluate the bounds using three metrics:
\emph{Normalized bound value}: To ensure comparability across environments with different reward scales, squared discounted return prediction errors are normalized by the maximum return observed during training. A value close to zero implies the bound closely follows the actual returns.
\emph{Tightness}: The difference between the predicted bound and the actual test error; smaller values indicate more accurate estimates of the discounted return prediction error.
\emph{Correlation}: A linear correlation is expected between the risk certificates and the observed test errors across policy instances.

We evaluate and compare the final posterior loss $\rho$ on full training data and held-out test data, alongside the corresponding PAC-Bayes bounds across all methods and environments. To mitigate overfitting common in continuous-control settings, where consecutive samples are highly correlated, we apply a thinning strategy that reduces redundancy while preserving data diversity. Full details for each experiment are provided in 
\cref{appsec:expdetails}.
We provide an implementation at \texttt{anonymous}.

\subsection{Results} 
Comprehensive results for every environment, policy instance, and repetition are presented in 
\cref{app:results}. 
The main text focuses on aggregated results.

\begin{figure*}
    \centering
    \adjustbox{max width=1\textwidth}{
    \includegraphics{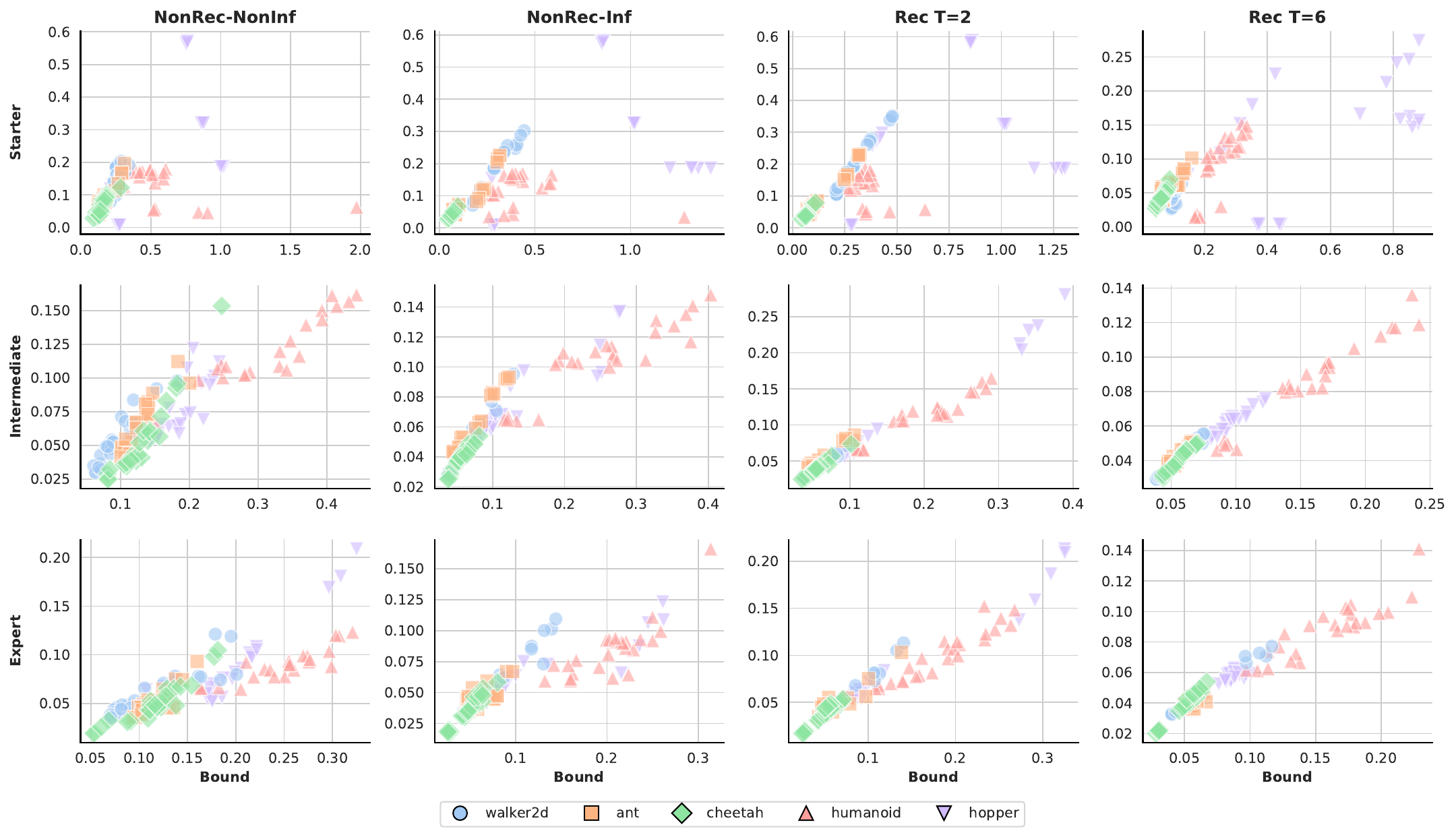}
    }
    \caption{
    \emph{Correlation between bounds and test errors.} PAC-Bayes bounds (x-axis) are plotted axis against true test errors (y-axis) for REDQ across five MuJoCo environments, policy instances, and repetitions to visualize correlation. We observe a high correlation, especially as policies improve and bounds become recursive.
     }
    \label{fig:scatter-plot}\vspace{-1em}
\end{figure*}

\paragraph{Q1: Can the test-time return of a policy be predicted with high precision?}

\cref{fig:scatter-plot} presents scatter plots of all PAC-Bayesian bounds discussed in \ref{sec:exp:design}, alongside policy instances and repetitions, against their respective test set errors across environments and levels of policy expertise for REDQ on five MuJoCo environments.
For every bound, the correlation between the bound and the test error increases with policy expertise.
Within a fixed expertise level, the correlation also improves as the bound becomes more advanced, a trend already evident for noisier starter policies. For example, in the brittle Hopper environment, which exhibits the weakest correlations overall, moving from NonRec-NonInf to Rec with $T=6$ raises the Pearson correlation from $0.4$ to $0.65$.
At higher expertise levels, the recursive bounds achieve correlations above $0.9$ in almost all environments.
Overall, there is a strong correlation between bounds and test errors.
There is also increasing scatter as the expertise level decreases. This is expected, as the effects of an unconverged policy on environment dynamics are less predictable.
The bounds therefore provide a reliable prediction of the test-time return, thereby answering Q1.
See the appendix for the corresponding correlation plots for the other agent and benchmark combinations.

\begin{center}
    \begin{tikzpicture}
    \node[
        draw=orange!70,
        fill=orange!10,
        line width=1pt,
        rounded corners=3mm,
        inner sep=5pt,
        text width=0.95\columnwidth,
        align=center
    ] {
        \textbf{Answer:} \emph{Bounds and test errors are strongly correlated.}
    };
    \end{tikzpicture}
\end{center}

\begin{figure*}
    \centering
    \includegraphics[width=0.99\textwidth]{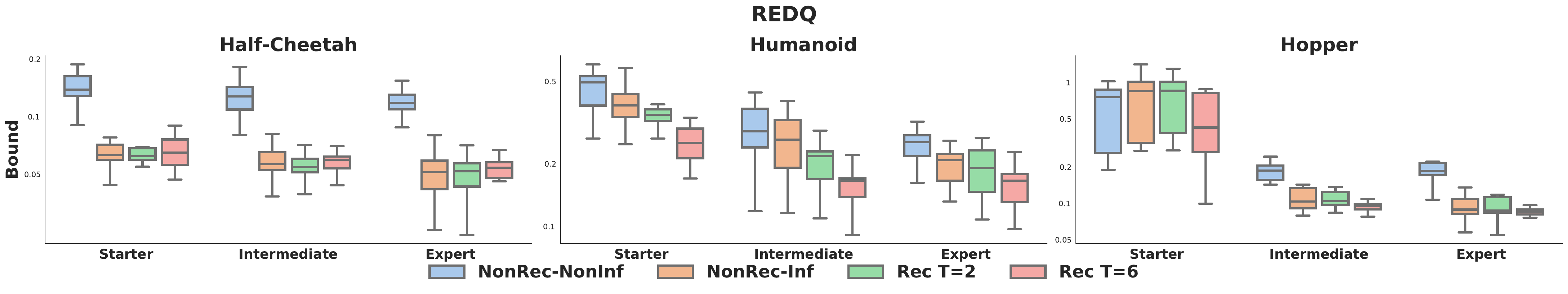}\\[1em]
    \includegraphics[width=0.99\textwidth]{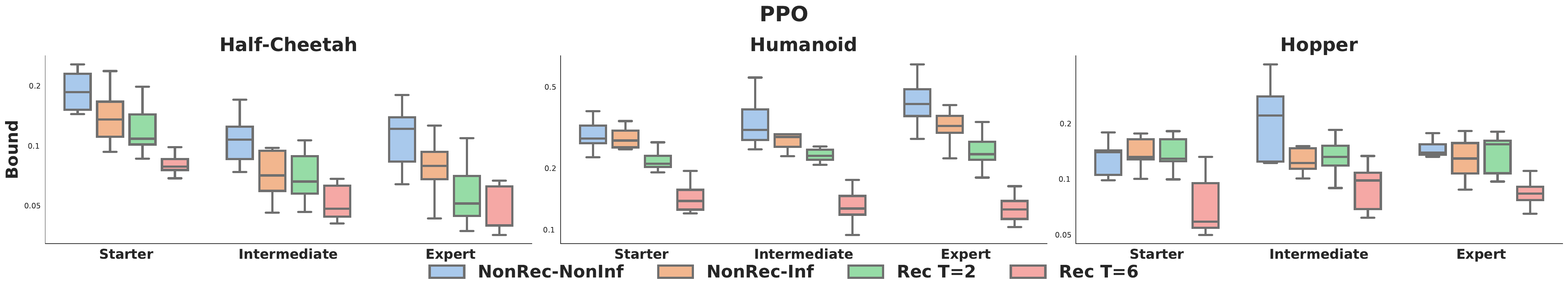}\\[1em]
    \includegraphics[width=0.99\textwidth]{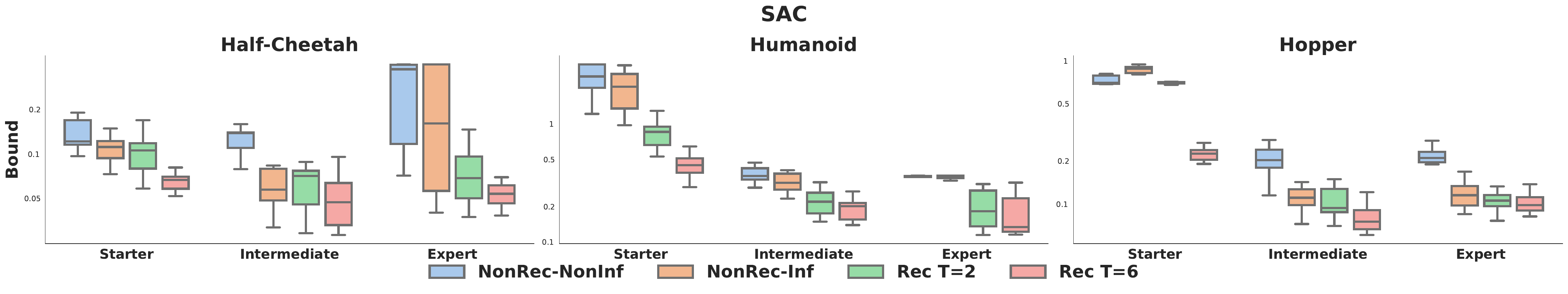}    
    \caption{
    \emph{Bound values.} Normalized bound values for all baselines across three MuJoCo environments and all policy quality levels. Results are aggregated across all seeds and repetitions.
}
    \label{fig:bound_mujoco}
\end{figure*}

\paragraph{Q2: What is the influence of a PAC-Bayes bound's structure?}
In \cref{fig:bound_mujoco} we plot the normalized bounds aggregated over policy instances and repetitions for each of the three models on three MuJoCo environments.
Data-informed priors usually improve bounds across environments for all policy levels and actor-critic methods. Introducing recursion further tightens bounds, with deeper recursion generally yielding the tightest results. 

\begin{center}
    \begin{tikzpicture}
    \node[
        draw=orange!70,
        fill=orange!10,
        line width=1pt,
        rounded corners=3mm,
        inner sep=5pt,
        text width=0.95\columnwidth,
        align=center
    ] {
        \textbf{Answer:} \emph{Bound tightness improves with increasing recursive depth.}
    };
    \end{tikzpicture}
\end{center}

\begin{figure}
  \centering
    \includegraphics[width=0.5\columnwidth]{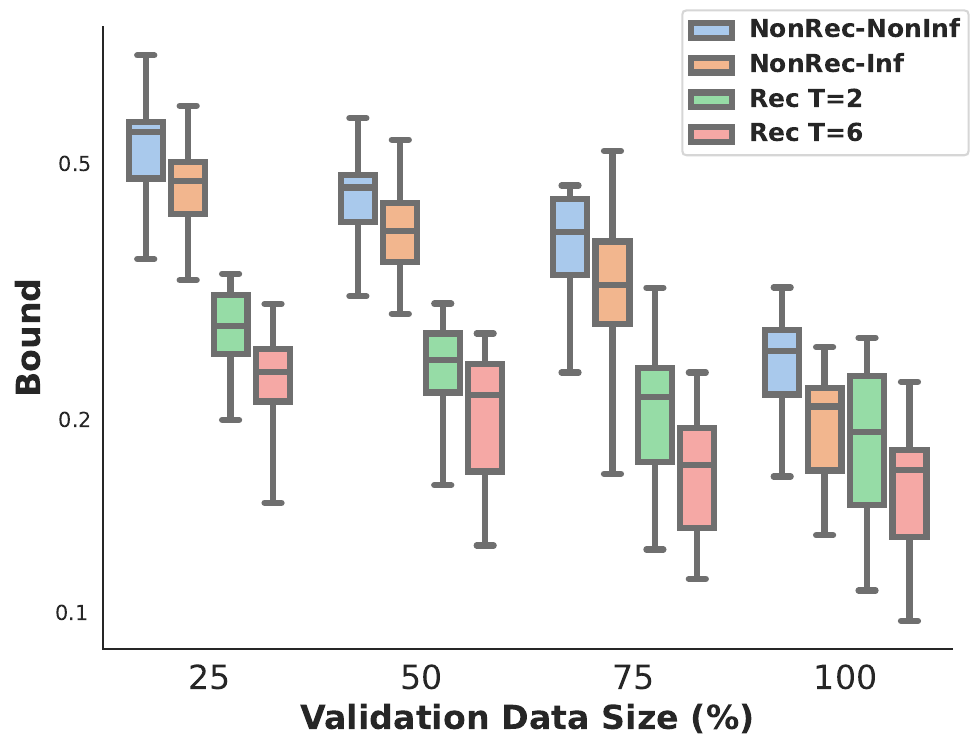}
        \caption{
        \emph{Effect of validation data size on tightness.} 
        Bounds for REDQ on Humanoid with an expert policy. 
        }
    \label{fig:validation_size}
\end{figure}

\paragraph{Q3: How does the validation set size influence the tightness of the risk certificate guarantee?}
Collecting validation data from physical robots is often costly. Hence, the sample efficiency of a risk-certificate generation pipeline is of particular interest. \cref{fig:validation_size} presents the tightness scores of the bounds across varying validation set sizes for REDQ in the Humanoid environment, while the test set size is kept fixed.
As expected, larger validation sets yield tighter bounds, with the effect most pronounced for the proposed recursive bounds. A recursion with depth $T=6$ achieves tightness comparable to that of the nonrecursive bounds while requiring only half as many data points.
These findings demonstrate that recursive bounds substantially improve sample efficiency, addressing Q3.

\begin{center}
    \begin{tikzpicture}
    \node[
        draw=orange!70,
        fill=orange!10,
        line width=1pt,
        rounded corners=3mm,
        inner sep=5pt,
        text width=0.95\columnwidth,
        align=center
    ] {
    \textbf{Answer:} \emph{Recursion improves sample efficiency.}
    };
    \end{tikzpicture}
\end{center}

\begin{figure}
  \centering
    \includegraphics[width=0.5\columnwidth]{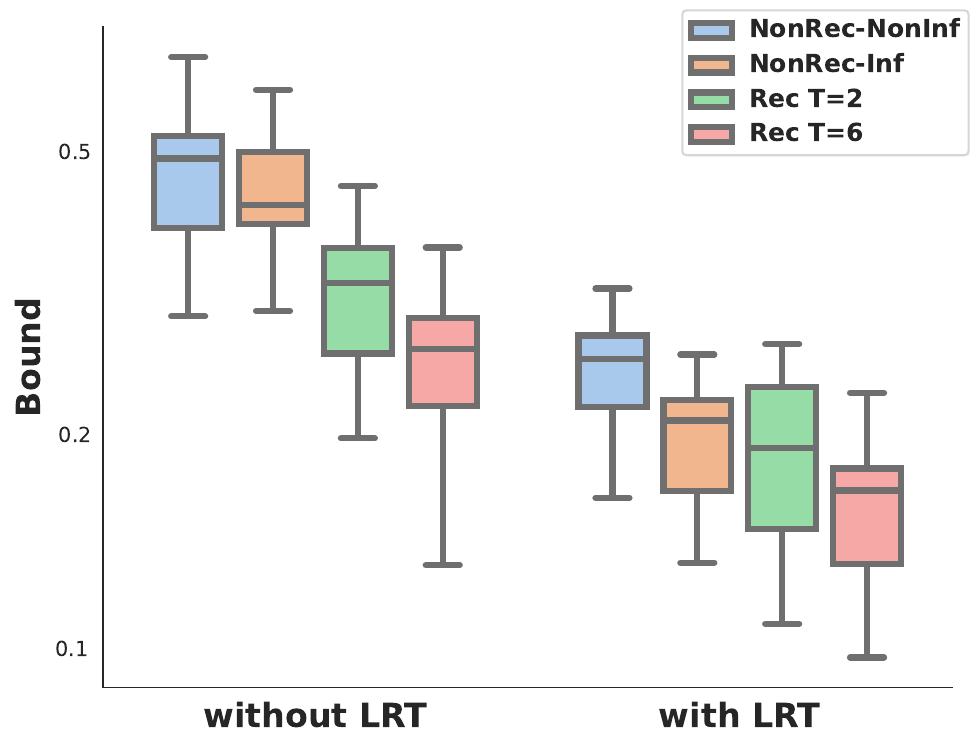}
    \vfill
    \caption{\emph{Local reparameterization trick (LRT).} Influence for REDQ on Humanoid with an expert policy.}
    \label{fig:local_reparam}
\end{figure}

\paragraph{Local reparameterization improves tightness.} 
To train our model, we use a Bayesian neural network (BNN) that represents uncertainty by learning distributions over parameters.
To our knowledge, prior work on PAC-Bayesian risk certification with BNNs has relied exclusively on \citet{pmlr-v37-blundell15}'s \emph{Bayes by backprop} approach \citep[see, e.g.,][]{perez2021tighter}.
We show in \cref{fig:local_reparam} that using the \emph{local reparameterization trick (LRT)} \citep{kingma2015variational} to compute the empirical risk term in the bound calculation substantially improves the tightness of all four evaluated bounds.
This effect holds even for the already saturated expert-level policy in the challenging Humanoid environment.
Further details can be found in 
\cref{app:results}.

\section{Limitations, Future Work, and Broader Impact}\label{sec:conclusion}

We restricted our empirical investigation to three actor-critic algorithms on three benchmark suites. This was a deliberate choice to facilitate interpretation and maintain feasibility. We do not expect meaningful additional information from extending the same pipeline to additional RL agents and benchmarks. The next major step would be to implement our pipeline on a physical platform under controlled conditions. The applicability of our findings to more advanced control settings, such as sparse-reward scenarios that require goal-conditioned or hierarchical RL algorithm design, is subject to further investigation. We leave this enterprise to future work, as deep learning-based solutions for such setups have not yet reached sufficient maturity to move beyond simulations.
Another significant advance would be to proceed from our current self-certified policy evaluation approach to self-certified policy optimization in an online setting. This would necessitate training the policy via a PAC-Bayes bound. However, RL is a feedback-loop system in which ensuring convergence, numerical stability, and optimal trade-offs between exploration and exploitation are major determinants of stable training. While promising preliminary results exist \citep{tasdighi2023pac,tasdighi2024deep}, the problem is fundamental and requires a dedicated research program—an effort that goes beyond the scope of a single paper.

Our work contributes to the trustworthy development of agentic AI technologies, thereby promoting their adoption by society. Public concerns about such technologies will be even more pronounced when they are deployed on physical systems in direct contact with humans. Thanks to reliable risk certificates, such safety-critical technologies are likely to receive wider adoption. This, in turn, will further accelerate their development by expanding the pool of practice and observations.

%% file: appendix.tex
\appendix

\section{Theory}

\subsection{PAC-Bayes-$\kl$ bound}
Assuming the definitions in \cref{par:notation}, we cite the following bound and $\kl$-inequality. See the respective references for further details.
\begin{theorem}[PAC-Bayes-$\kl$ bound \citep{seeger2002pac,maurer2004note}]\label{thm:pac_bayes_kl}
For any probability distribution $\rho_0 \in \mcP$ that is independent of $\mcD$ and any $\delta \in (0,1)$, we have 
\begin{equation*}
\mdP\Bigl(\exists \rho \in \mcP: \kl\bigl(\mdE_{\rho}[\hat L(h)]\Vert \Ep{\rho}{L(h)}\bigr) \geq \big(\KL{\rho}{\rho_0} + \ln(2\sqrt{N}/\delta)\big)/N\Bigr) \leq \delta.
\end{equation*}
\end{theorem}
\begin{proof}
    See, e.g., \citet{maurer2004note} for a proof of the bound.
\end{proof}

\begin{lemma}[$\kl$-inequality \citep{langford2005tutorial,foong2021tight,foong2022note}]\label{lem:kl-inequality}
Let $Z_1, \ldots, Z_N$ be i.i.d.\ random variables taking values in an interval $[0,1]$, and let $\E{Z_n}=p$ for all $n$. Let their empirical mean be $\hat{p}=\frac{1}{N}\sum_{n=1}^N Z_n$. Then, for any $\delta \in (0,1)$, we have

  \begin{equation*}
      \mdP\left(\kl(\hat{p}\Vert p)\geq \ln(1/\delta)/N\big.\right) \leq \delta,
  \end{equation*}
  the inverse of which is given by
  \begin{equation*}
       \mdP\left(p\geq \invklup\left(\hat{p},\ln(1/\delta)/N\right)\!\big.\right)\leq \delta,
       \end{equation*}
   and
    \begin{equation*}
       \mdP\left(p\leq \invkllow\left(\hat{p},\ln(1/\delta)/N\right)\!\big.\right)\leq \delta.
  \end{equation*}
\end{lemma}
\begin{proof}
    See \citet{langford2005tutorial}, Corollary 3.7, for a proof of the bound.
\end{proof}

\subsection{Split-$\kl$ inequality}
Let $Z\in[a,b]$, with $a,b \in \mdR$, be a random variable, and let $p = \E{Z}$. For $\mu \in [a,b]$, define ${Z^+ =\max\{0,Z-\mu\}}$ and $Z^- = \max\{0,\mu - Z\}$ so that $Z = \mu + Z^+ - Z^-$. Let $p^+ = \E{Z^+}$ and $p^- = \E{Z^-}$ be their respective expectations, and let ${\hat p^+ = \frac1N\sum_{n=1}^NZ_n^+}$ and $\hat p^- = \frac1N\sum_{n=1}^NZ_n^-$ be their empirical means for an i.i.d.\ sample $Z_1,\ldots, Z_N$. 

\begin{lemma}[Split-$\kl$ inequality \citep{wu2022split}]\label{lem:split-kl}
For any $\mu\in[a,b]$ and $\delta \in (0,1)$,
\begin{equation*}
    \mdP \Bigg(p \leq \mu + (b - \mu)\invklup\left(\frac{\hat p^+}{b - \mu},\frac{\ln(2/\delta)}{N}\right) - (\mu - a)\invkllow\Big(\frac{\hat p^-}{\mu-a},\frac{\ln(2/\delta)}{N}\Big)\!\!\Bigg) \geq 1 - \delta.
\end{equation*}
\end{lemma}
\begin{proof}
    The lemma follows by applying \cref{lem:kl-inequality} to each of the $\kl$ terms and a union bound.
\end{proof}

\section{Related work} \label{sec:related_work}

\paragraph{Theoretical analysis of reinforcement learning.}
Exploration in finite-horizon Markov decision processes (MDPs) has been addressed in several prior works. UCBVI \citep{azar2017minimax} offers a simple approach based on optimism under uncertainty, combining computational efficiency with conceptual clarity for exploration in finite-horizon MDPs. BayesUCBVI \citep{tiapkin2022dirichlet} builds on this by introducing a tabular, stage-dependent, episodic reinforcement learning algorithm that uses a quantile of the posterior distribution of Q-values for optimism, eliminating the need for explicit bonus terms. However, it remains uncertain whether the PSRL \citep{osband2013more} approach can achieve optimal problem-independent regret bounds. UCRL2 \citep{auer2008near}, an extension of UCRL \citep{auer2006logarithmic}, employs upper confidence bounds to guide policy selection, encouraging exploration of uncertain state–action pairs and reducing regret over time, resulting in a near-optimal regret bound that grows logarithmically with the number of episodes. \citet{agrawal2017optimistic} propose a PSRL-inspired algorithm for multi-armed bandits, using posterior sampling to incorporate uncertainty and to form an optimistic policy, achieving near-optimal regret bounds. It remains unclear whether it can achieve the problem-independent lower bound. OPSRL \citep{tiapkin2022optimistic} extends PSRL by using multiple posterior samples instead of a single one, providing high-probability regret-bound guarantees and ensuring strong performance.

\paragraph{PAC-Bayes analysis.}
PAC-Bayesian analysis provides a frequentist framework for integrating prior knowledge into learning algorithms \citep{shawe1997pac, mcallester1998some, alquier2024user}. It leverages priors that sustain high confidence during learning \citep{mcallester1999pac, seeger2002pac, catoni2007pac, thiemann2017strongly, rivasplata2019pac, perez2021tighter}, but the resulting confidence intervals rely on data excluded from prior formation, making prior data allocation a critical challenge. Strategies for crafting effective PAC-Bayesian priors address this limitation \citep{ambroladze2006tighter,dziugaite21ontherolw,wu2024recursive}.

\paragraph{PAC-Bayesian risk certificate.} 
In recent years, PAC-Bayesian methods have also been used to compress neural nets \citep{lotfi22pacbayes} and to provide risk certificates for both uninformed \citep{dziugaite2017computing} and data-informed priors \citep{dziugaite2018data,dziugaite21ontherolw,perez2021tighter}.
\citet{reeb2018learning} introduce PAC-Bayesian bounds for Gaussian process-based regression, and an increasing body of literature now provides risk certificates for deep generative models, such as VAEs \citep{mbake2023statistical}, GANs \citep{mbacke23generalization}, and contrastive learning frameworks \citep{nozawa2020contrastive}.
\citet{hsuren2022slr} use PAC-Bayes to provide realistic safety guarantees, focusing primarily on worst-case risk and out-of-distribution safety.

\paragraph{PAC-Bayesian RL.}
PAC analysis in RL has led to algorithms that provide formal guarantees on sample complexity and performance. \citet{strehl2006pac} introduced Delayed Q-Learning, an efficient model-free RL method. \citet{strehl2009reinforcement} established PAC bounds for both model-free and model-based methods in finite MDPs, highlighting sample efficiency. For complex observation spaces, \citet{krishnamurthy2016pac} proposed Least-Squares Value Elimination, extending contextual bandit frameworks to sequential settings with PAC guarantees.
In model-based RL, \citet{dann2017unifying} introduced the Uniform-PAC framework for finite-state episodic MDPs, achieving optimal sample complexity and regret bounds. \citet{jiang2018pac} improved model-based efficiency with a polynomial-complexity algorithm.
PAC-Bayesian frameworks have also been applied in RL for policy evaluation and stability. \citet{fard2011pac} introduced a PAC-Bayesian method for policy evaluation with probabilistic guarantees, which was subsequently extended to soft actor–critic approaches \citep{tasdighi2023pac} and deep exploration with sparse rewards \citep{tasdighi2024deep}.
These methods are designed exclusively for policy optimization, using PAC-Bayesian principles as algorithmic guidelines rather than for reliable risk certificate generation, which is our objective. Both methods employ approximately calculated PAC-Bayesian bounds without quantifying the resulting approximation errors. While such approximations are acceptable in policy search contexts, where learning performance is the main concern, they preclude rigorous analytical statements about test-time performance and therefore cannot support risk certification.
Despite these advancements, the tightness, that is, the quality, of risk certificates remained an open question.

\section{Algorithms}
    In this section, two pseudocodes are presented. \cref{alg:pacbayes_workflow} illustrates the overall pipeline of the method, capturing the main components from policy training to PAC-Bayes bound calculation.
    
    \cref{alg:rec_bound} presents the Recursive PAC-Bayes framework in detail, outlining the step-by-step calculation of the recursive bound across different splits of the training data, ultimately yielding the final bound. This framework supports multiple levels of recursion (e.g., two or six levels in our experiments). Details on the size of each data split and the selection strategy are provided in \cref{appsec:expdetails}.

\begin{algorithm}
\caption{Predicting Policy Return via PAC-Bayes Bound Fitting}
\label{alg:pacbayes_workflow}
\begin{algorithmic}[1]
\STATE {\bf Input:} Environment, actor-critic algorithm, BNN architecture, PAC-Bayes bounds
\STATE {\bf Output:} Predicted test-time discounted returns and PAC-Bayes generalization bounds

\vspace{0.5em}
\noindent {\bf Policy Training}
\STATE Train policy $\pi$ using the REDQ actor-critic algorithm.
\STATE Save policy parameters after specified training steps.

\vspace{0.5em}
\noindent {\bf Data Collection}
\STATE Disable further learning.
\STATE Execute the saved policy for a few roll-outs and collect data.
\STATE Split the data into training and test sets.

\vspace{0.5em}
\noindent {\bf Bayesian Model Training}
\STATE Initialize a Bayesian neural network (BNN).
\STATE Train the Bayesian model on the training data using a PAC-Bayes-inspired loss as the training objective.

\vspace{0.5em}
\noindent {\bf Bound Construction and Evaluation}
\FOR{each bound type}
   \STATE Specify prior and posterior distributions.
   \STATE Compute the PAC-Bayes bound using model-predicted outputs.
\ENDFOR

\vspace{0.5em}
\STATE {\bf Return:} Model predictions on the test data and bound evaluations on the training data.
\end{algorithmic}
\end{algorithm}

\begin{algorithm}
\caption{Recursive PAC-Bayes bound loss computation}
\label{alg:rec_bound}
\begin{algorithmic}[1]

\STATE {\bf Input:} Training dataset split $\mathcal{D} = S_1 \cup \cdots \cup S_T$ with total size $N = |\mathcal{D}|$, where $N_t = |S_{\geq t}|$; scaling factors $\kappa_1, \ldots, \kappa_T$; posterior parameters $\{\rho_1, \ldots, \rho_T\}$; and fixed confidence levels $\delta$, $\delta'$.

\STATE Calculate the loss for the first portion: $\hat{L}(h) \in [0, B]$ 
\begin{equation*}
       \mdE_{\rho_1}[\hat L^+(h)] \leq \kl^{-1,+} \left( \frac{1}{N_t} \sum \left( \max \left\{0, \hat{L}(h) - \mu \right\} \right), \frac{\ln(T/\delta')}{N_t}
       \right).
\end{equation*}

\STATE Specify the first portion of the recursive bound:
\begin{equation*}
   \mcB_1(\rho_1) = B\invklup\biggl(\frac{\mdE_{\rho_1}[\hat L^+(h)]}{B},\frac{\KL{\rho_1}{\rho_0^*} + \ln(2T\sqrt{N}/\delta)}{N}\biggr).
\end{equation*}

\FOR{each portion from $t=2$ to the end}
\STATE Calculate the excess loss:
\begin{equation*}
      L^\text{exc}_t(h) = L(h) - \kappa_t\Ep{\rho_{t-1}}{L(h')} \in [-\kappa_t B, B].
\end{equation*}

\STATE Then,
\begin{equation*}
      \mdE_{\rho_t}[\hat{L}^{\text{exc}+}_t(h)] \leq  \kl^{-1,+} \left( \frac{1}{N_t} \sum \left( \max\left\{0, \hat{L}^{\text{exc}}_t(h) - \mu \right\} \right), \frac{\ln(2T/\delta')}{N_t}
       \right),
\end{equation*}
\begin{equation*}
      \mdE_{\rho_t}[\hat{L}^{\text{exc}-}_t(h)] \leq  \kl^{-1,-} \left( \frac{1}{N_t} \sum \left( \max\left\{0, \mu - \hat{L}^{\text{exc}}_t(h) \right\} \right), \frac{\ln(2T/\delta')}{N_t}
       \right).
\end{equation*}

\STATE Compute $\mcE_t(\rho_t)$:
\begin{align*}
   \mcE_t(\rho_t) &= \mu + (B-\mu)\invklup\biggl(\frac{\mdE_{\rho_t}[\hat{L}^{\text{exc}+}_t(h)]}{B - \mu},\frac{\KL{\rho_t}{\rho_{t-1}^*} + \ln(4T\sqrt{N_t}/\delta)}{N_t} \biggr)\\
   &\quad - (\mu + \kappa_t B)\invkllow\biggl(\frac{\mdE_{\rho_t}[\hat{L}^{\text{exc}-}_t(h)]}{\mu + \kappa_t B},\frac{\KL{\rho_t}{\rho_{t-1}^*} + \ln(4T\sqrt{N_t}/\delta)}{N_t} \biggr).
\end{align*}

\STATE Update the bound iteratively:
\begin{equation*}
   \mcB_t(\rho_t) = \mcE_t(\rho_t,\kappa_t) + \kappa_t \mcB_{t-1}(\rho^*_{t-1}).
\end{equation*}

\STATE $t = t+1$
\ENDFOR

\STATE With probability at least $1 - \delta$,
\begin{equation*}
    \Ep{\rho_t}{L(h)} \leq \mcB_t(\rho_t).
\end{equation*}
   
\end{algorithmic}
\end{algorithm}

\section{Experimental Details}\label{appsec:expdetails}

In this section, we present the details and design choices for our environments, dataset, model training, and the replication of our experimental results. We consider five environments from version ‘v4’ of the MuJoCo suite \citep{todorov2012mujoco} (ant, half-cheetah, hopper, walker2d, humanoid), three environments from DM Control \citep{tassa2018deepmind} (ball-in-cup, reacher, walker), and three from Meta-World \citep{yu2020meta} (drawer-open, window-open, reach). All implementations in this work utilize the PyTorch framework \citep{paszke2019pytorch}, version ‘2.5.1’, and the OpenAI Gym environment \citep{brockman2016openai}, version ‘0.29.0’.\\

\subsection{Specific hyperparameters} 
For recursive PAC-Bayes, we use $\kappa_t = 1/2$ for all $t$ and study the impact of recursion for data split over different portions. In our experiments, we set $\delta = 0.025$, following \citet{wu2024recursive}'s implementation.\footnote{\url{https://github.com/pyijiezhang/rpb}}

In addition to the primary confidence parameter $\delta$, which controls the overall probability with which the PAC-Bayes bound holds, we introduce an auxiliary confidence parameter $\delta'$, following \citet{wu2024recursive}. This parameter accounts for the approximation error introduced when estimating the expected empirical loss and empirical excess loss via sampling \citep{perez2021tighter}, since these quantities lack closed-form expressions in our setting. Specifically, $\delta'$ ensures that these sample-based estimates are accurate with high probability. We set $\delta' = 0.01$, consistent with the choice in the referenced work, and apply a union bound to combine the confidence levels. As a result, the overall bound holds with probability at least $1 - \delta - \delta'$, covering both the generalization guarantee and the estimation accuracy. This consideration is applied only during the evaluation phase for estimating the bounds, not during the optimization process.

\subsection{Constructing validation and test datasets}

Our methodology comprises two phases: \emph{training} and \emph{evaluation}. To collect the transition data required for constructing validation and test datasets, we train an actor-critic agent (SAC, PPO, or REDQ). 

Training is conducted for \num{300000} environment steps, after which the resulting policy is saved as the \textit{expert} policy. We also snapshot the policy at \num{100000} and \num{200000} steps, referring to them as the \textit{starter} and \textit{intermediate} policies, respectively, to enable analysis across varying levels of agent proficiency.

In the evaluation phase, learning is disabled, and the agent is executed with exploration turned off and no policy updates performed. Using the frozen policy (at each of the three proficiency levels), we collect 100 episodes of state transitions and rewards to construct a validation dataset for fitting the PAC-Bayes bound. An additional 100 episodes are collected under the same conditions to form the test dataset. This test data is used to compute a proxy for generalization performance by evaluating the predicted discounted returns of the frozen policy on previously unseen trajectories.

\subsection{Neural network architectures} \label{NNarch}

\paragraph{REDQ agent training.}
We use an ensemble of ten critic networks and a single actor network. Each neural network consists of three layers, with layer normalization after each layer to regularize the network, following the approach of \citet{ball2023efficient}. Additionally, we use the concatenated ReLU activation function \citep{shang16understanding}, which combines the positive and negative parts of two ReLU activations and concatenates them.
This leads to richer feature representations and enables the learning of more complex patterns.

During training, we employ \num{10000} warm-up steps without policy updates, during which the agent only interacts with the environment; during the evaluation phase, we do not use any warm-up steps.
We use a replay ratio (RR) of one in both the training and evaluation phases, which provides training stability and sample efficiency. We employ the Huber loss in the critic to measure the discrepancy between predicted Q-values and target Q-values, as it consistently provides performance advantages across various baselines. Furthermore, the temperature parameter $\alpha$ is automatically tuned during training to regulate policy entropy and balance exploration and exploitation, following the method introduced by \citet{haarnoja2018soft-app}.

All these architectural and training choices were empirically found to improve the model’s overall performance. \cref{tab:hyperparameters} summarizes the hyperparameters and network configurations used in our experiments.

\textbf{PPO agent training.} Proximal Policy Optimization is an on-policy algorithm that trains a single actor network together with a single critic network representing the state-value function. The policy is updated using a clipped surrogate objective to prevent overly large updates, which improves stability but makes PPO less sample-efficient than off-policy methods. In our experiments, we followed the standard PPO baseline and evaluated performance on three widely used MuJoCo environments of varying difficulty and dimensionality: \texttt{HalfCheetah, Humanoid}, and \texttt{Hopper}. Since PPO typically requires more training to converge, we extended the training horizon to 1,000,000; 3,000,000; and 5,000,000 steps to obtain starter, intermediate, and expert policies, respectively. Risk bounds were computed in the same manner as in the REDQ experiments.

\textbf{SAC agent training.} Soft Actor-Critic is an off-policy algorithm that combines a single actor network with two Q-function critic networks, following the standard practice to mitigate overestimation bias. SAC is generally more sample-efficient than on-policy methods and stabilizes learning through target networks and soft updates. We implemented SAC in its canonical form and considered three training levels: 50,000 steps (starter), 100,000 steps (intermediate), and 1,000,000 steps (expert), using the same environments as above. The corresponding policies were then evaluated under our PAC-Bayesian risk bounds, consistent with the REDQ setting.

For both methods, the actor and critic networks follow the same neural network architecture described in \cref{NNarch} for REDQ.

\paragraph{Bayesian model training.} 
We aim to predict the test-time discounted return of a policy $\pi$ by fitting a PAC-Bayes generalization bound using training data. To this end, we use a Bayesian neural network (BNN), which allows us to capture model uncertainty and compute valid generalization guarantees.

The BNN is a feedforward neural network with two hidden layers and one output layer. Each layer is implemented as a variational Bayesian linear layer, where weights are modeled as independent Gaussian distributions with learnable means and variances. ReLU activations are used between layers to introduce nonlinearity.

During training, we apply the local reparameterization trick \citep{kingma2015variational}, which samples the layer outputs rather than the weights. This reduces the variance of the gradient estimates, improves training stability, and reduces computational cost.

We train the BNN using a PAC-Bayes-inspired objective based on McAllester's bound \citep{mcallester1998some}, which balances empirical loss and a KL divergence term. The KL term acts as a regularizer that penalizes deviation from a prior distribution, effectively controlling model complexity. This enables us to compute reliable generalization bounds and estimate the expected test-time return of the policy. Architectural details are summarized in \cref{tab:hyperparameters}.

\subsection{Further details on the bounds} 
We evaluate two non-recursive baselines and two recursive variants with different depths. 

For the uninformed baseline (NonRec-NonInf), we use a PAC-Bayes-kl inverse bound in which the prior is initialized without dependence on the training data. The posterior is learned using the entire training dataset, and the bound is evaluated on this full dataset.
For the informed case (NonRec-Inf), we split the training dataset in half: the first half is used to learn a data-informed prior, and the second to learn the posterior, with both steps minimizing a PAC-Bayes bound. The bound is evaluated using the same subset used to learn the posterior.

In the recursive settings, we apply these splits iteratively. For example, with recursion depth two, we divide the data into two 50–50 splits. For deeper recursion with six steps, we split the training set into increasingly smaller partitions based on the number of episodes $[3, 4, 6, 12, 25, 50]$. This approach allocates fewer data points for prior learning early on while reserving more for later refinement, enabling progressively more targeted fine-tuning.

Like \citet{wu2024recursive}, we apply a relaxation of the PAC-Bayes-kl bound in our experiments by optimizing the McAllester bound \citep{mcallester1998some}, while the bound evaluation is conducted using the split kl-based bound. 
We always utilize a set of factorized Gaussian distributions to represent the priors and posteriors associated with every trainable parameter within the classifiers. These distributions take the form \(\pi = \mathcal{N}(w, \sigma I)\), where \(w \in \mathbb{R}^d\) is the mean vector and \(\sigma\) denotes the scalar variance parameter. Initially, an uninformed prior \(\pi_0 = \mathcal{N}(w_0, \sigma_0 I)\) is employed, which does not depend on the training data. Here, the mean parameter is randomly initialized using Kaiming uniform initialization, while the log-variance parameter is initialized to a fixed value of (-4.6).
We train all models using the hyperparameters provided in \cref{tab:hyperparameters}, specified with the header 'PAC-Bayes bound fitting'.

\begin{table}
\centering
\caption{\emph{Hyperparameters used in our experiments.} This includes both hyperparameters applied to the Actor-Critic setup in agent training and Bayesian model training for PAC-Bayes bound fitting.}
\vspace{1em}
\label{tab:hyperparameters}
\adjustbox{max width=0.98\textwidth}{
\begin{tabular}{lcc}
\toprule
\textbf{Agent training: } & \\
\midrule
Evaluation episodes (evaluation mode) & 1 \\
Evaluation frequency (evaluation mode) & 1 \\
Evaluation episodes (training mode) & - \\
Evaluation frequency (training mode) & - \\
Discount factor $(\gamma)$ & 0.99 \\
$n$-step returns & 1 step \\
Replay ratio & 1 \\
Number of critic networks & 10 \\
Replay buffer size & 100{,}000 \\
Maximum timesteps$^{*}$ & 300{,}000 \\ 
Number of hidden layers for all networks & 2 \\
Number of hidden units per layer & 256 \\
Nonlinearity & CReLU \\
Mini-batch size  & 256 \\
Network regularization method & Layer normalization (LN) \citep{ball2023efficient} \\
Actor/critic optimizer & Adam \citep{kingma2014adam} \\
Optimizer learning rates $(\eta_{\phi}, \eta_{\theta})$ & 3e-4 \\
Polyak averaging parameter $(\tau)$ & 5e-3 \\
\midrule
\textbf{PAC-Bayes bound fitting:} & \\
\midrule
Number of hidden layers & 2 \\
Number of hidden units per layer & 256 \\
Nonlinearity & ReLU \\
Network regularization method & Layer normalization (LN) \citep{ball2023efficient} \\
BNN optimizer & Adam \citep{kingma2014adam}\\
Optimizer learning rate & 2e-2 \\
Gradient clipping type & max-norm \\
Gradient clipping threshold & 1 \\
Learning rate scheduler & StepLR\\
Learning rate decay factor & 0.5\\
Scheduler step size (epochs) & 10 \\

\bottomrule
\end{tabular}}
\end{table}

\paragraph{Mitigate Sample Correlation.}
To reduce overfitting caused by high correlation among consecutive samples—where states are highly similar and their associated target values nearly identical—a thinning strategy is employed. This strategy selectively samples fewer data points to weaken temporal dependence. Specifically, in environments such as \texttt{Ant} and \texttt{Half-Cheetah}, which tend to produce longer episodes, every fifth sample is retained. In contrast, for environments with shorter episodes, such as \texttt{Humanoid}, \texttt{Hopper}, and \texttt{Walker2d}, every third sample is retained. This approach reduces redundancy while ensuring that the dataset still contains sufficient information for training.

\paragraph{Effect of the Local Reparameterization Trick}
We incorporate the Local Reparameterization Trick (LRT) into our variational Bayesian neural network architecture, where, to our knowledge, it has not been used in previous Bayesian approaches to reinforcement learning. Unlike standard sampling methods that draw one set of weights per layer and propagate forward, LRT enables sampling at the level of individual pre-activations, which significantly reduces the variance of gradient estimates during training. This, in turn, leads to more stable optimization and an improved posterior. To evaluate its impact, we compare the tightness of all our baselines on the Humanoid environment for two versions of our model, with and without LRT. As shown in \cref{fig:local_reparam}, applying LRT consistently yields tighter bounds.
Results are aggregated over five policy instances and five repetitions each, considering only the expert-level policy. In both cases, the results show a clear improvement. These expert-policy improvements suggest gains at other policy levels as well.

\subsection{Computational requirements}
Experiments were conducted on a single computer equipped with a GeForce RTX 4090 GPU, an Intel(R) Core(TM) i7-14700K CPU (5.6 GHz), and 96 GB of memory. Training five policy instances for REDQ to convergence in each environment requires approximately 30 minutes per instance, totaling 150 minutes. Collecting validation and test episodes requires around 20 minutes per policy level, totaling 60 minutes. Model training and PAC-Bayesian bound computation across five policy instances, five repetitions, four baselines, and three policies take four minutes per run, totaling approximately \num{1200} minutes per environment.

\section{Further results} \label{app:results}

\subsection{Additional REDQ Bounds}
We evaluate REDQ on two additional MuJoCo environments (\cref{fig:mujoco_redq_rest}), as well as on two further benchmarks-DM Control (\cref{fig:dm_redq}) and Meta-World (\cref{fig:mw_redq})-and observe consistent qualitative results.

\begin{figure}
    \centering
    \includegraphics[width=0.98\linewidth]{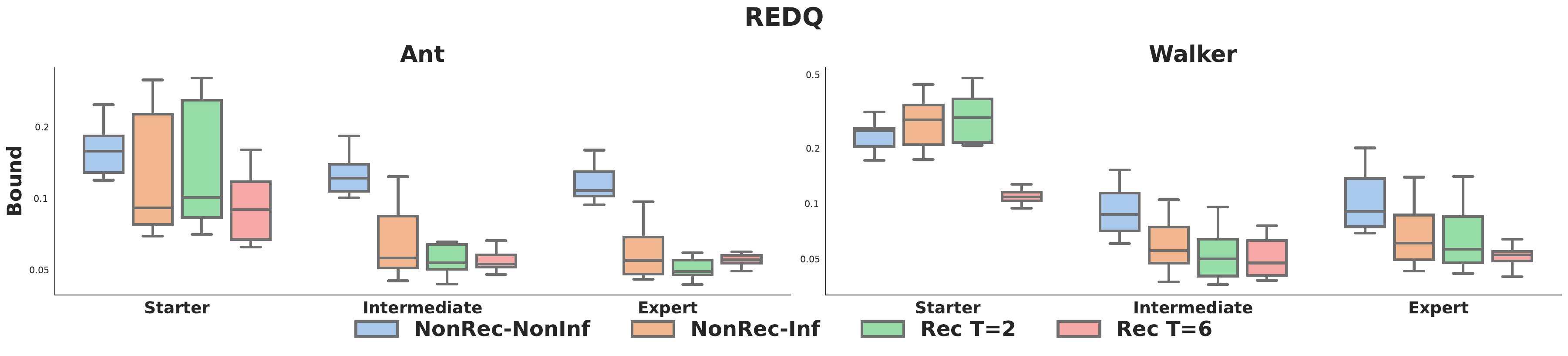}
    \caption{\emph{Bound values.} Normalized bound values for REDQ on two MuJoCo environments and all policy quality levels. Results are aggregated across all seeds and repetitions.}
    \label{fig:mujoco_redq_rest}
\end{figure}

\begin{figure}
    \centering
    \includegraphics[width=0.98\linewidth]{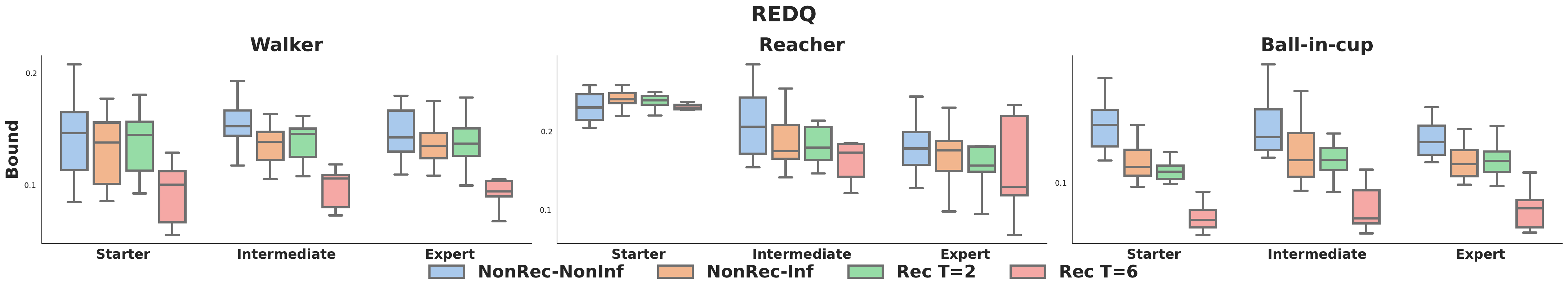}
    \caption{\emph{Bound values.} Normalized bound values for REDQ on three DM Control environments and all policy quality levels. Results are aggregated across all seeds and repetitions.}
    \label{fig:dm_redq}
\end{figure}

\begin{figure}
    \centering
    \includegraphics[width=0.98\linewidth]{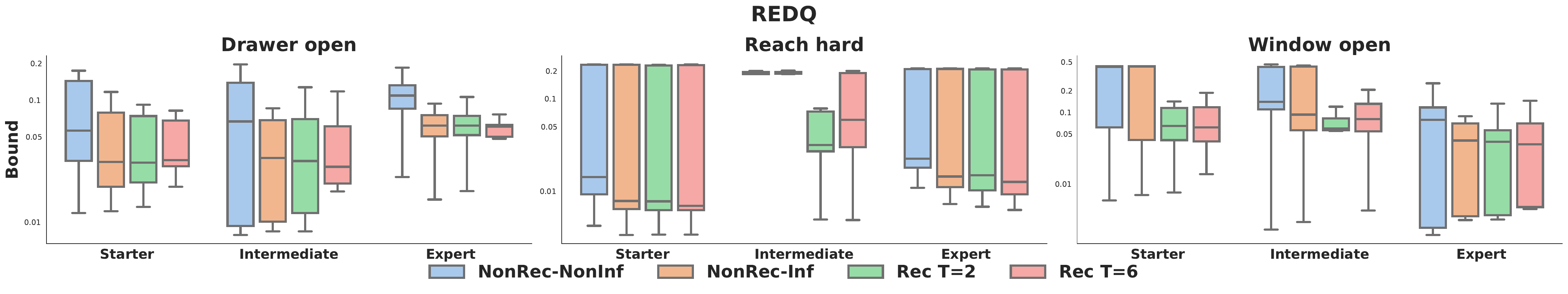}
    \caption{\emph{Bound values.} Normalized bound values for REDQ on three Meta-world environments and all policy quality levels. Results are aggregated across all seeds and repetitions.}
    \label{fig:mw_redq}
\end{figure}

\subsection{Correlation plots for PPO and SAC}
The corresponding plots to \cref{fig:scatter-plot} for PPO and SAC are provided in \cref{fig:scatter-plot-ppo} and \cref{fig:scatter-plot-sac}, respectively, which exhibit a similar correlation structure.

\begin{figure*}
    \centering
    \adjustbox{max width=1\textwidth}{
    \includegraphics{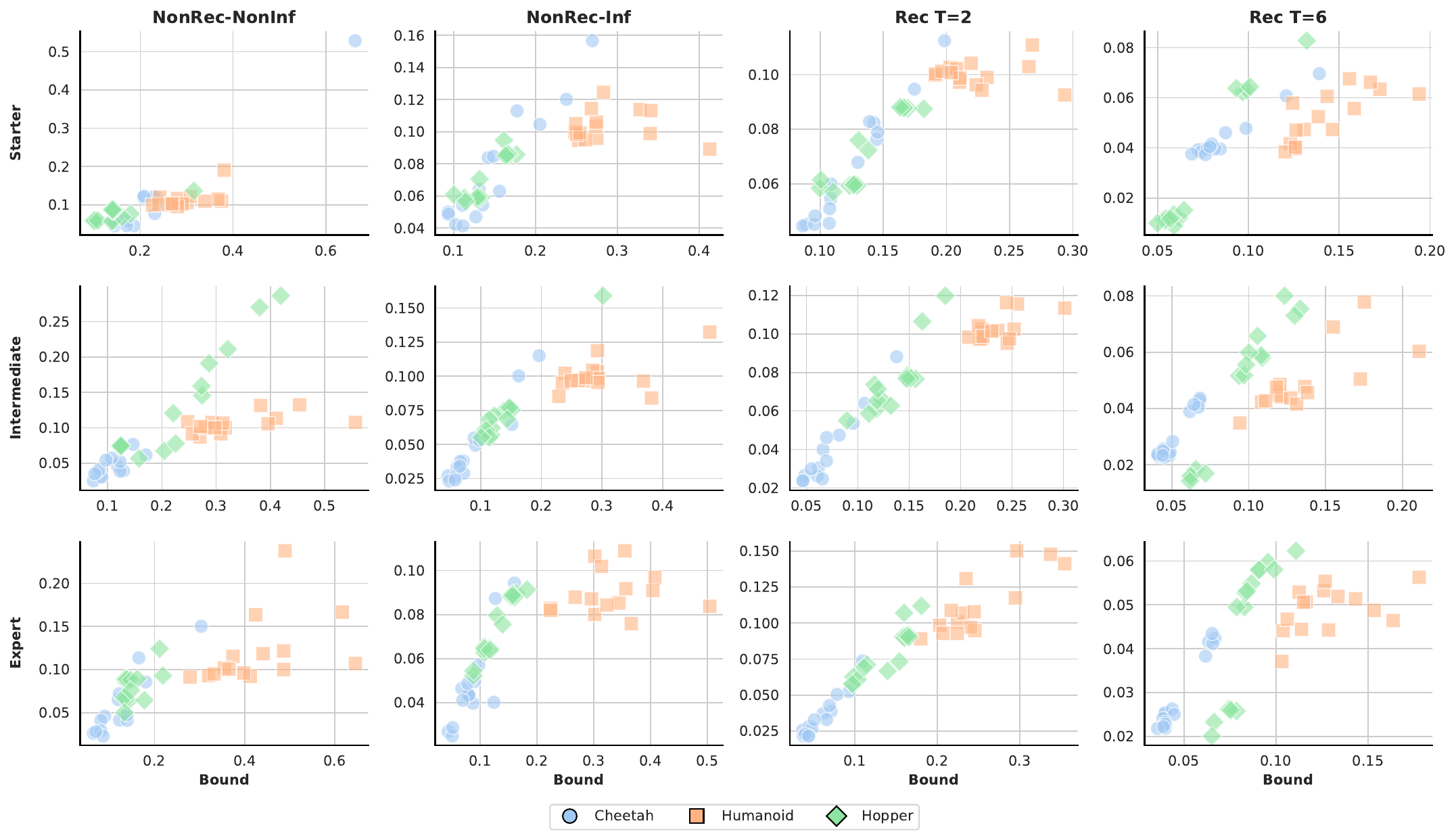}
    }
    \caption{
    \emph{Correlation between bounds and test errors.} PAC-Bayes bounds (x-axis) are plotted axis against true test errors (y-axis) for PPO across three MuJoCo environments, policy instances, and repetitions to visualize correlation. We observe a high correlation, especially as policies improve and bounds become recursive.
     }
    \label{fig:scatter-plot-ppo}\vspace{-1em}
\end{figure*}

\begin{figure*}
    \centering
    \adjustbox{max width=1\textwidth}{
    \includegraphics{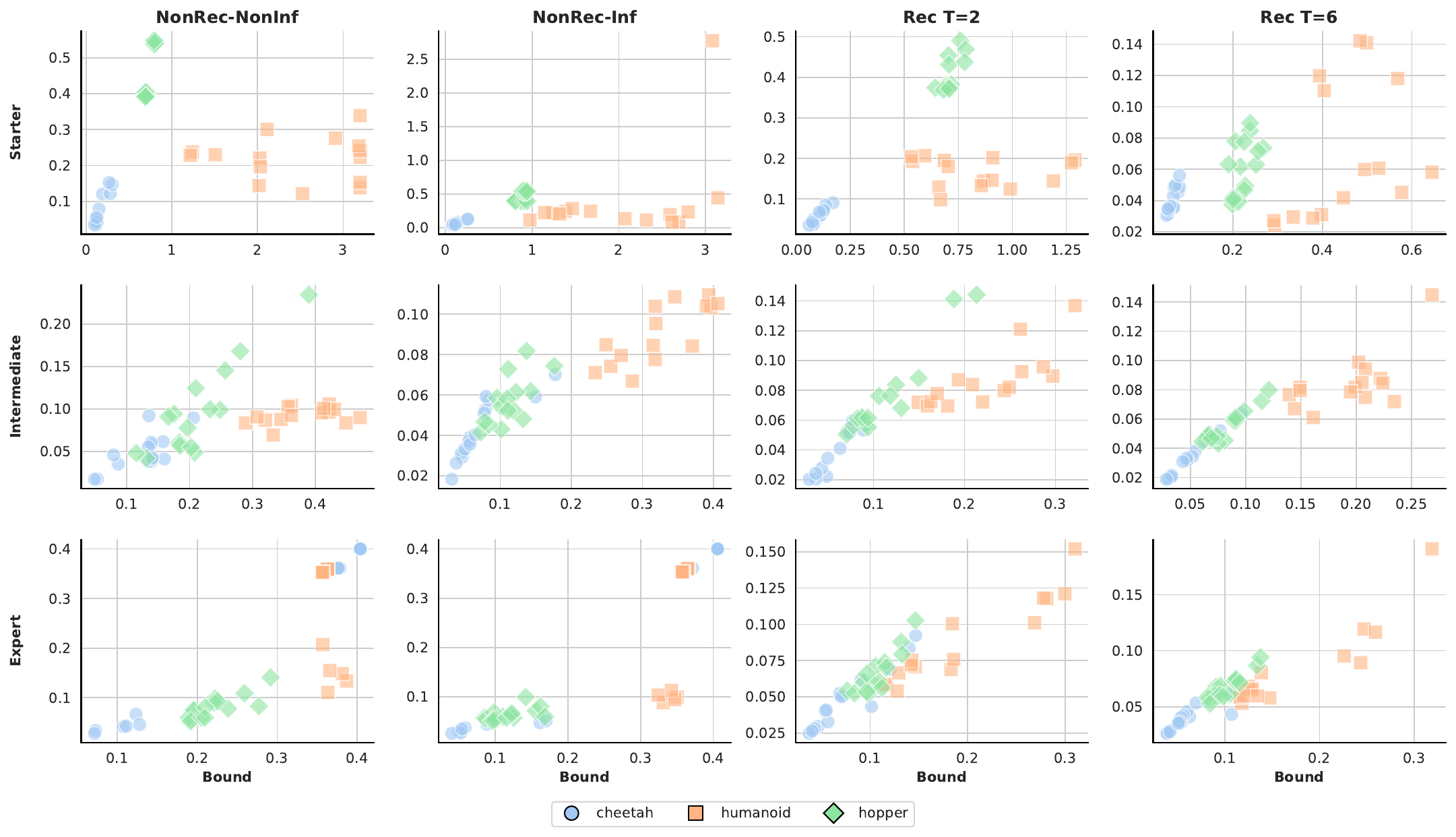}
    }
    \caption{
    \emph{Correlation between bounds and test errors.} PAC-Bayes bounds (x-axis) are plotted axis against true test errors (y-axis) for SAC across three MuJoCo environments, policy instances, and repetitions to visualize correlation. We observe a high correlation, especially as policies improve and bounds become recursive.
     }
    \label{fig:scatter-plot-sac}\vspace{-1em}
\end{figure*}

\subsection{Tightness of the bounds} 
\cref{fig:difference_mujoco_redq,fig:difference_dm_redq,fig:difference_mw_redq,fig:difference_mujoco_ppo,fig:difference_mujoco_sac} show the tightness of all baselines, measured as the difference between the normalized bound score and the normalized test error. Values are aggregated across policy instances and repetitions for each of the five environments, with smaller values indicating tighter bounds. Box plots display the distribution of these differences: the box covers the interquartile range (25th to 75th percentile), the median is marked inside, and whiskers extend to points within 1.5 times the interquartile range, illustrating typical variability. Outliers beyond this range are omitted for clarity. This effectively captures both the central tendency and the dispersion of the data across policies and methods.
The plot indicates that data-informed priors generally yield tighter bounds, although this effect is less pronounced at the Starter level. Moreover, applying recursion and increasing its depth further tighten the bounds. The improvements due to recursion are more noticeable in challenging environments and less so in simpler tasks.

\begin{figure}
    \centering
    \includegraphics[width=0.95\textwidth]{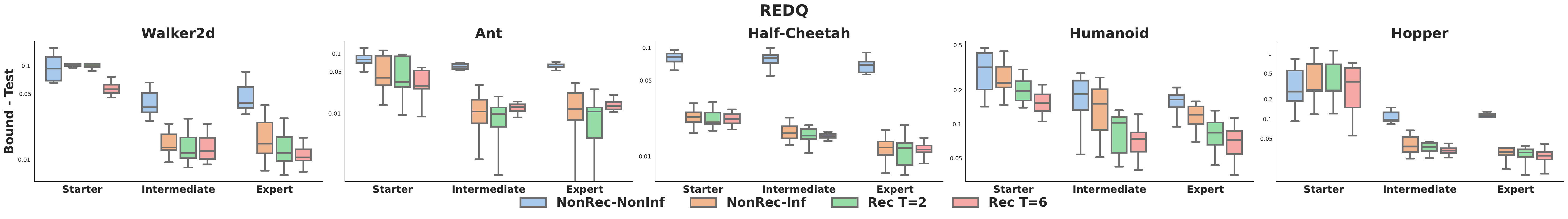}    
    \caption{
    \emph{(Bound - test) values.} Tightness of the bound over all baselines and policy levels considered for REDQ across five MuJoCo environments. The plots aggregated over policy instances and their repetitions. 
    }
    \label{fig:difference_mujoco_redq}
\end{figure}
\begin{figure}
    \centering
    \includegraphics[width=0.95\textwidth]{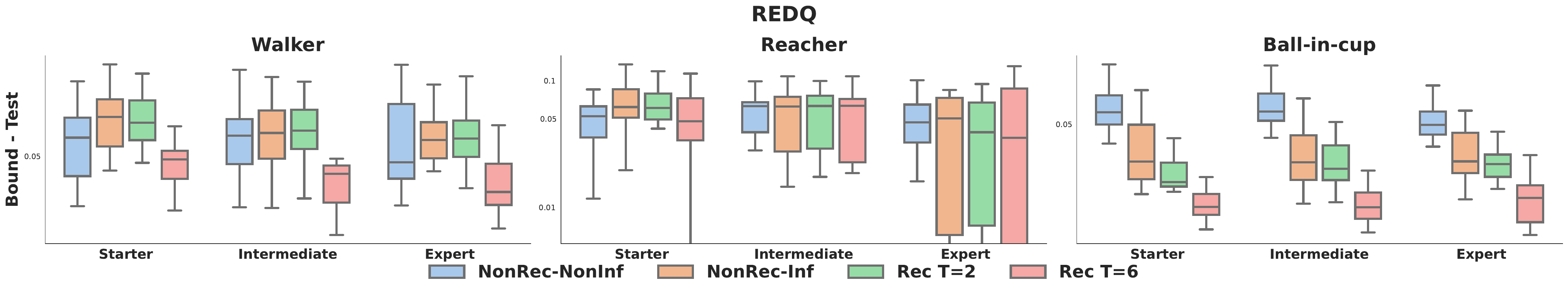}    
    \caption{
    \emph{(Bound - test) values.} Tightness of the bound over all baselines and policy levels considered for REDQ across three DM Control environments. The plots aggregated over policy instances and their repetitions. 
    }
    \label{fig:difference_dm_redq}
\end{figure}
\begin{figure}
    \centering
    \includegraphics[width=0.95\textwidth]{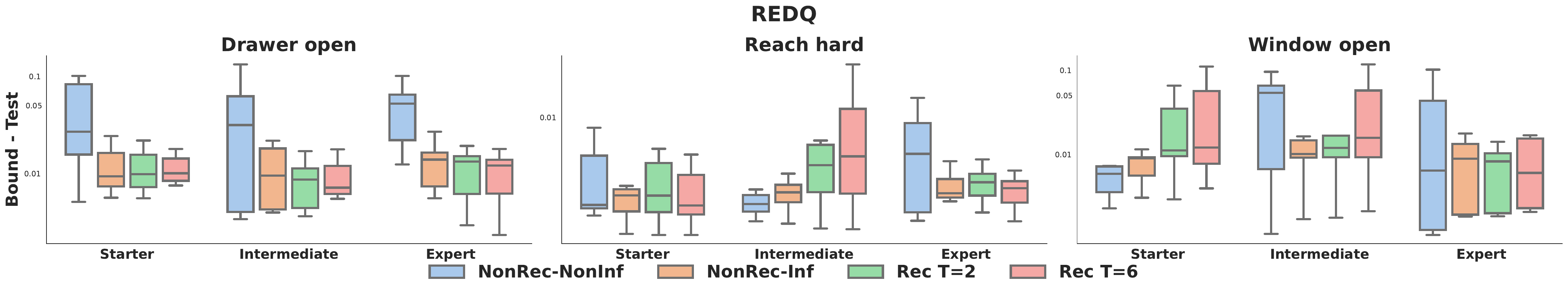}    
    \caption{
    \emph{(Bound - test) values.} Tightness of the bound over all baselines and policy levels considered for REDQ across three Meta-world environments. The plots aggregated over policy instances and their repetitions. 
    }
    \label{fig:difference_mw_redq}
\end{figure}
\begin{figure}
    \centering
    \includegraphics[width=0.95\textwidth]{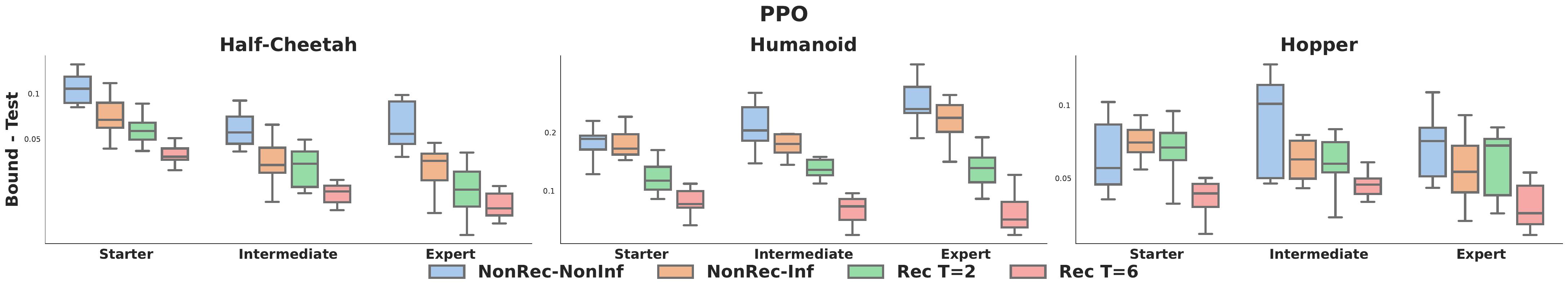}    
    \caption{
    \emph{(Bound - test) values.} Tightness of the bound over all baselines and policy levels considered for PPO across three MuJoCo environments. The plots aggregated over policy instances and their repetitions. 
    }
    \label{fig:difference_mujoco_ppo}
\end{figure}
\begin{figure}
    \centering
    \includegraphics[width=0.95\textwidth]{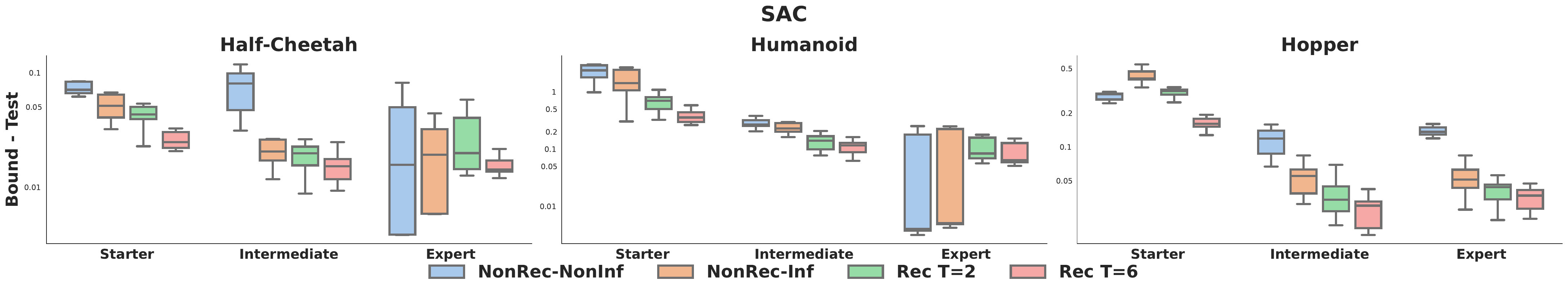}    
    \caption{
    \emph{(Bound - test) values.} Tightness of the bound over all baselines and policy levels considered for SAC across three MuJoCo environments. The plots aggregated over policy instances and their repetitions. 
    }
    \label{fig:difference_mujoco_sac}
\end{figure}

\subsection{Validation data size} 
\Cref{app:pearson_correlation} examines the Pearson correlation between normalized bound scores and normalized test errors across policy levels and varying validation data sizes for the Humanoid environment—our most challenging setting due to its high-dimensional state and action spaces. Each heatmap displays a point estimate of these correlations for all baselines across different validation set sizes for a specific policy, aggregated over policy instances and repetitions. Although correlations are generally weaker and more variable under the Starter policy, they tend to improve with stronger policies and larger validation sets for intermediate- and expert-level policies. Recursive bounds, especially those with greater recursion depth, consistently show stronger positive correlations. These findings emphasize that both policy strength and validation data size affect the bound’s ability to predict generalization, with recursion providing the most reliable alignment between bounds and test errors.

\begin{figure}
    \centering
\includegraphics[width=1\textwidth]{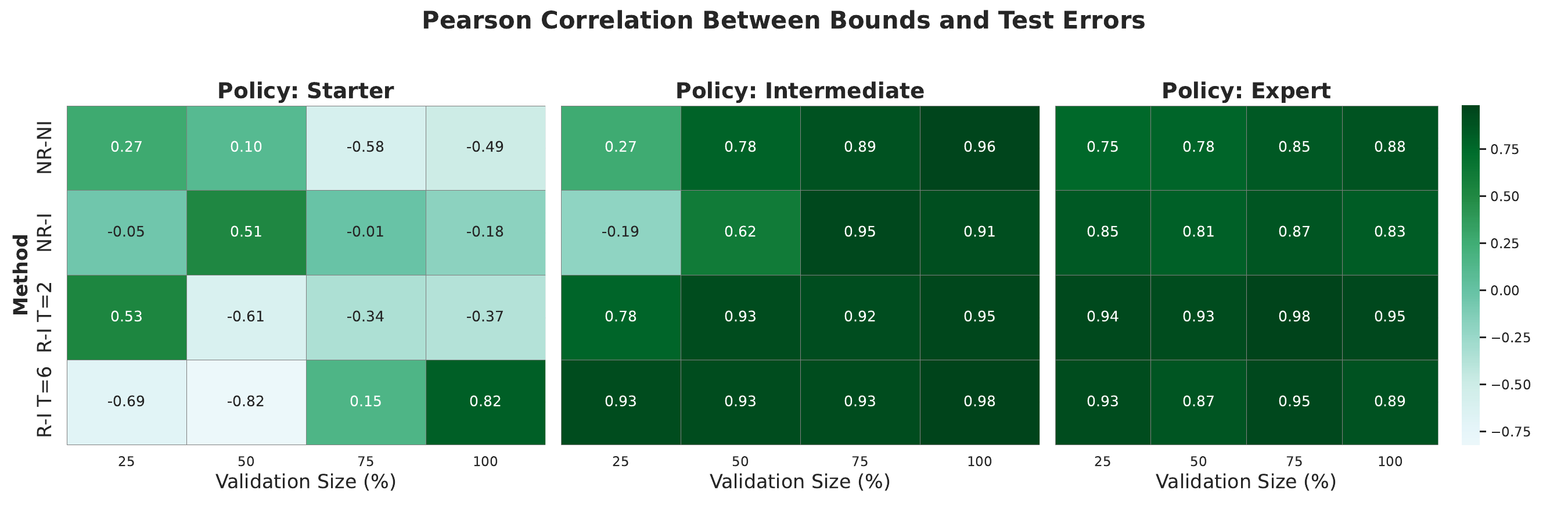}
    \caption{\emph{Pearson correlation.} Correlation between bound scores and test error considering five policy instances and five repetitions for each, across various validation sizes and policy levels.
    }
\label{app:pearson_correlation}
\end{figure}

\newpage
\subsection{Raw numbers}
We include the raw individual values averaged over five repetitions in all scenarios.

\begin{table}
\centering
\caption{Normalized train error, test error, and bound for REDQ on MuJoCo's walker2d}
\begin{subtable}{\textwidth}
\centering
\caption{Starter policy}
\adjustbox{max width=0.9\textwidth}{
\begin{tabular}{lccccccccccccccc}
\toprule
 & \multicolumn{3}{c}{\textsc{Seed 01}} & \multicolumn{3}{c}{\textsc{Seed 02}} & \multicolumn{3}{c}{\textsc{Seed 03}} & \multicolumn{3}{c}{\textsc{Seed 04}} & \multicolumn{3}{c}{\textsc{Seed 05}} \\
\textsc{Method} & Train & Test & Bound & Train & Test & Bound & Train & Test & Bound & Train & Test & Bound & Train & Test & Bound \\
\cmidrule(lr){2-4}\cmidrule(lr){5-7}\cmidrule(lr){8-10}\cmidrule(lr){11-13}\cmidrule(lr){14-16}
NonRec-NonInf & 0.150 & 0.142 & 0.292 & 0.103 & 0.103 & 0.173 & 0.182 & 0.186 & 0.252 & 0.158 & 0.161 & 0.254 & 0.097 & 0.097 & 0.221 \\
NonRec-Inf & 0.283 & 0.272 & 0.416 & 0.104 & 0.104 & 0.208 & 0.183 & 0.187 & 0.286 & 0.236 & 0.240 & 0.341 & 0.096 & 0.095 & 0.194 \\
Rec T=2 & 0.360 & 0.347 & 0.475 & 0.105 & 0.105 & 0.210 & 0.189 & 0.193 & 0.292 & 0.269 & 0.273 & 0.369 & 0.158 & 0.156 & 0.245 \\
Rec T=6 & 0.063 & 0.059 & 0.117 & 0.031 & 0.030 & 0.101 & 0.045 & 0.047 & 0.105 & 0.059 & 0.061 & 0.114 & 0.060 & 0.060 & 0.108 \\
\bottomrule
\end{tabular}}
\vspace{0.5em}
\end{subtable}
\begin{subtable}{\textwidth}
\centering
\caption{Intermediate policy}
\adjustbox{max width=0.9\textwidth}{
\begin{tabular}{lccccccccccccccc}
\toprule
 & \multicolumn{3}{c}{\textsc{Seed 01}} & \multicolumn{3}{c}{\textsc{Seed 02}} & \multicolumn{3}{c}{\textsc{Seed 03}} & \multicolumn{3}{c}{\textsc{Seed 04}} & \multicolumn{3}{c}{\textsc{Seed 05}} \\
\textsc{Method} & Train & Test & Bound & Train & Test & Bound & Train & Test & Bound & Train & Test & Bound & Train & Test & Bound \\
\cmidrule(lr){2-4}\cmidrule(lr){5-7}\cmidrule(lr){8-10}\cmidrule(lr){11-13}\cmidrule(lr){14-16}
NonRec-NonInf & 0.046 & 0.046 & 0.083 & 0.056 & 0.056 & 0.086 & 0.046 & 0.046 & 0.083 & 0.054 & 0.054 & 0.095 & 0.071 & 0.069 & 0.134 \\
NonRec-Inf & 0.034 & 0.034 & 0.046 & 0.042 & 0.042 & 0.057 & 0.043 & 0.043 & 0.059 & 0.038 & 0.038 & 0.052 & 0.073 & 0.072 & 0.098 \\
Rec T=2 & 0.037 & 0.037 & 0.048 & 0.032 & 0.032 & 0.042 & 0.033 & 0.033 & 0.044 & 0.042 & 0.042 & 0.056 & 0.063 & 0.061 & 0.084 \\
Rec T=6 & 0.031 & 0.031 & 0.040 & 0.031 & 0.031 & 0.042 & 0.040 & 0.041 & 0.057 & 0.039 & 0.039 & 0.052 & 0.054 & 0.053 & 0.072 \\
\bottomrule
\end{tabular}}
\vspace{0.5em}
\end{subtable}
\begin{subtable}{\textwidth}
\centering
\caption{Expert policy}
\adjustbox{max width=0.9\textwidth}{
\begin{tabular}{lccccccccccccccc}
\toprule
 & \multicolumn{3}{c}{\textsc{Seed 01}} & \multicolumn{3}{c}{\textsc{Seed 02}} & \multicolumn{3}{c}{\textsc{Seed 03}} & \multicolumn{3}{c}{\textsc{Seed 04}} & \multicolumn{3}{c}{\textsc{Seed 05}} \\
\textsc{Method} & Train & Test & Bound & Train & Test & Bound & Train & Test & Bound & Train & Test & Bound & Train & Test & Bound \\
\cmidrule(lr){2-4}\cmidrule(lr){5-7}\cmidrule(lr){8-10}\cmidrule(lr){11-13}\cmidrule(lr){14-16}
NonRec-NonInf & 0.088 & 0.086 & 0.181 & 0.039 & 0.040 & 0.073 & 0.048 & 0.048 & 0.090 & 0.051 & 0.051 & 0.092 & 0.071 & 0.071 & 0.119 \\
NonRec-Inf & 0.090 & 0.091 & 0.130 & 0.045 & 0.047 & 0.060 & 0.039 & 0.039 & 0.052 & 0.056 & 0.056 & 0.075 & 0.046 & 0.046 & 0.061 \\
Rec T=2 & 0.084 & 0.084 & 0.114 & 0.039 & 0.040 & 0.049 & 0.040 & 0.040 & 0.051 & 0.043 & 0.042 & 0.054 & 0.064 & 0.064 & 0.079 \\
Rec T=6 & 0.076 & 0.072 & 0.106 & 0.034 & 0.035 & 0.043 & 0.040 & 0.040 & 0.050 & 0.041 & 0.041 & 0.053 & 0.043 & 0.043 & 0.053 \\
\bottomrule
\end{tabular}}
\vspace{0.5em}
\end{subtable}
\end{table}

\begin{table}
\centering
\caption{Normalized train error, test error, and bound for REDQ on MuJoCo's ant}
\begin{subtable}{\textwidth}
\centering
\caption{Starter policy}
\adjustbox{max width=0.9\textwidth}{
\begin{tabular}{lccccccccccccccc}
\toprule
 & \multicolumn{3}{c}{\textsc{Seed 01}} & \multicolumn{3}{c}{\textsc{Seed 02}} & \multicolumn{3}{c}{\textsc{Seed 03}} & \multicolumn{3}{c}{\textsc{Seed 04}} & \multicolumn{3}{c}{\textsc{Seed 05}} \\
\textsc{Method} & Train & Test & Bound & Train & Test & Bound & Train & Test & Bound & Train & Test & Bound & Train & Test & Bound \\
\cmidrule(lr){2-4}\cmidrule(lr){5-7}\cmidrule(lr){8-10}\cmidrule(lr){11-13}\cmidrule(lr){14-16}
NonRec-NonInf & 0.074 & 0.070 & 0.143 & 0.061 & 0.067 & 0.147 & 0.112 & 0.085 & 0.177 & 0.070 & 0.068 & 0.124 & 0.195 & 0.156 & 0.281 \\
NonRec-Inf & 0.051 & 0.043 & 0.080 & 0.052 & 0.057 & 0.073 & 0.137 & 0.106 & 0.216 & 0.066 & 0.064 & 0.094 & 0.268 & 0.217 & 0.309 \\
Rec T=2 & 0.060 & 0.053 & 0.085 & 0.056 & 0.060 & 0.079 & 0.198 & 0.160 & 0.256 & 0.072 & 0.071 & 0.101 & 0.283 & 0.230 & 0.320 \\
Rec T=6 & 0.049 & 0.040 & 0.068 & 0.048 & 0.055 & 0.065 & 0.078 & 0.061 & 0.115 & 0.063 & 0.061 & 0.088 & 0.103 & 0.086 & 0.139 \\
\bottomrule
\end{tabular}}
\vspace{0.5em}
\end{subtable}
\begin{subtable}{\textwidth}
\centering
\caption{Intermediate policy}
\adjustbox{max width=0.9\textwidth}{
\begin{tabular}{lccccccccccccccc}
\toprule
 & \multicolumn{3}{c}{\textsc{Seed 01}} & \multicolumn{3}{c}{\textsc{Seed 02}} & \multicolumn{3}{c}{\textsc{Seed 03}} & \multicolumn{3}{c}{\textsc{Seed 04}} & \multicolumn{3}{c}{\textsc{Seed 05}} \\
\textsc{Method} & Train & Test & Bound & Train & Test & Bound & Train & Test & Bound & Train & Test & Bound & Train & Test & Bound \\
\cmidrule(lr){2-4}\cmidrule(lr){5-7}\cmidrule(lr){8-10}\cmidrule(lr){11-13}\cmidrule(lr){14-16}
NonRec-NonInf & 0.057 & 0.055 & 0.129 & 0.070 & 0.089 & 0.149 & 0.046 & 0.056 & 0.116 & 0.046 & 0.049 & 0.113 & 0.055 & 0.065 & 0.123 \\
NonRec-Inf & 0.053 & 0.051 & 0.066 & 0.069 & 0.086 & 0.107 & 0.051 & 0.059 & 0.069 & 0.036 & 0.042 & 0.052 & 0.038 & 0.047 & 0.054 \\
Rec T=2 & 0.044 & 0.042 & 0.053 & 0.065 & 0.081 & 0.098 & 0.043 & 0.050 & 0.056 & 0.038 & 0.043 & 0.055 & 0.034 & 0.044 & 0.047 \\
Rec T=6 & 0.039 & 0.038 & 0.052 & 0.046 & 0.049 & 0.064 & 0.039 & 0.041 & 0.052 & 0.035 & 0.040 & 0.051 & 0.041 & 0.043 & 0.054 \\
\bottomrule
\end{tabular}}
\vspace{0.5em}
\end{subtable}
\begin{subtable}{\textwidth}
\centering
\caption{Expert policy}
\adjustbox{max width=0.9\textwidth}{
\begin{tabular}{lccccccccccccccc}
\toprule
 & \multicolumn{3}{c}{\textsc{Seed 01}} & \multicolumn{3}{c}{\textsc{Seed 02}} & \multicolumn{3}{c}{\textsc{Seed 03}} & \multicolumn{3}{c}{\textsc{Seed 04}} & \multicolumn{3}{c}{\textsc{Seed 05}} \\
\textsc{Method} & Train & Test & Bound & Train & Test & Bound & Train & Test & Bound & Train & Test & Bound & Train & Test & Bound \\
\cmidrule(lr){2-4}\cmidrule(lr){5-7}\cmidrule(lr){8-10}\cmidrule(lr){11-13}\cmidrule(lr){14-16}
NonRec-NonInf & 0.039 & 0.040 & 0.101 & 0.053 & 0.057 & 0.117 & 0.042 & 0.050 & 0.106 & 0.058 & 0.056 & 0.135 & 0.052 & 0.052 & 0.116 \\
NonRec-Inf & 0.041 & 0.043 & 0.057 & 0.046 & 0.050 & 0.059 & 0.039 & 0.047 & 0.049 & 0.053 & 0.050 & 0.081 & 0.037 & 0.036 & 0.052 \\
Rec T=2 & 0.035 & 0.037 & 0.050 & 0.041 & 0.045 & 0.049 & 0.039 & 0.047 & 0.049 & 0.069 & 0.068 & 0.098 & 0.038 & 0.037 & 0.049 \\
Rec T=6 & 0.040 & 0.040 & 0.053 & 0.043 & 0.044 & 0.057 & 0.044 & 0.043 & 0.057 & 0.040 & 0.038 & 0.059 & 0.040 & 0.040 & 0.052 \\
\bottomrule
\end{tabular}}
\vspace{0.5em}
\end{subtable}
\end{table}

\begin{table}
\centering
\caption{Normalized train error, test error, and bound for REDQ on MuJoCo's cheetah}
\begin{subtable}{\textwidth}
\centering
\caption{Starter policy}
\adjustbox{max width=0.9\textwidth}{
\begin{tabular}{lccccccccccccccc}
\toprule
 & \multicolumn{3}{c}{\textsc{Seed 01}} & \multicolumn{3}{c}{\textsc{Seed 02}} & \multicolumn{3}{c}{\textsc{Seed 03}} & \multicolumn{3}{c}{\textsc{Seed 04}} & \multicolumn{3}{c}{\textsc{Seed 05}} \\
\textsc{Method} & Train & Test & Bound & Train & Test & Bound & Train & Test & Bound & Train & Test & Bound & Train & Test & Bound \\
\cmidrule(lr){2-4}\cmidrule(lr){5-7}\cmidrule(lr){8-10}\cmidrule(lr){11-13}\cmidrule(lr){14-16}
NonRec-NonInf & 0.071 & 0.072 & 0.176 & 0.058 & 0.058 & 0.135 & 0.062 & 0.062 & 0.128 & 0.066 & 0.065 & 0.153 & 0.060 & 0.061 & 0.148 \\
NonRec-Inf & 0.043 & 0.043 & 0.066 & 0.039 & 0.039 & 0.063 & 0.037 & 0.037 & 0.059 & 0.047 & 0.047 & 0.070 & 0.042 & 0.042 & 0.066 \\
Rec T=2 & 0.042 & 0.042 & 0.066 & 0.044 & 0.044 & 0.066 & 0.037 & 0.037 & 0.057 & 0.051 & 0.050 & 0.076 & 0.039 & 0.039 & 0.060 \\
Rec T=6 & 0.056 & 0.057 & 0.081 & 0.037 & 0.038 & 0.059 & 0.032 & 0.032 & 0.053 & 0.050 & 0.049 & 0.074 & 0.041 & 0.041 & 0.062 \\
\bottomrule
\end{tabular}}
\vspace{0.5em}
\end{subtable}
\begin{subtable}{\textwidth}
\centering
\caption{Intermediate policy}
\adjustbox{max width=0.9\textwidth}{
\begin{tabular}{lccccccccccccccc}
\toprule
 & \multicolumn{3}{c}{\textsc{Seed 01}} & \multicolumn{3}{c}{\textsc{Seed 02}} & \multicolumn{3}{c}{\textsc{Seed 03}} & \multicolumn{3}{c}{\textsc{Seed 04}} & \multicolumn{3}{c}{\textsc{Seed 05}} \\
\textsc{Method} & Train & Test & Bound & Train & Test & Bound & Train & Test & Bound & Train & Test & Bound & Train & Test & Bound \\
\cmidrule(lr){2-4}\cmidrule(lr){5-7}\cmidrule(lr){8-10}\cmidrule(lr){11-13}\cmidrule(lr){14-16}
NonRec-NonInf & 0.066 & 0.066 & 0.149 & 0.068 & 0.067 & 0.152 & 0.033 & 0.033 & 0.093 & 0.044 & 0.044 & 0.117 & 0.062 & 0.062 & 0.148 \\
NonRec-Inf & 0.042 & 0.043 & 0.060 & 0.045 & 0.045 & 0.063 & 0.030 & 0.030 & 0.044 & 0.040 & 0.040 & 0.058 & 0.043 & 0.043 & 0.063 \\
Rec T=2 & 0.047 & 0.047 & 0.064 & 0.041 & 0.041 & 0.057 & 0.029 & 0.029 & 0.042 & 0.039 & 0.039 & 0.057 & 0.046 & 0.046 & 0.064 \\
Rec T=6 & 0.046 & 0.046 & 0.064 & 0.046 & 0.046 & 0.064 & 0.032 & 0.032 & 0.046 & 0.039 & 0.039 & 0.054 & 0.046 & 0.046 & 0.062 \\
\bottomrule
\end{tabular}}
\vspace{0.5em}
\end{subtable}
\begin{subtable}{\textwidth}
\centering
\caption{Expert policy}
\adjustbox{max width=0.9\textwidth}{
\begin{tabular}{lccccccccccccccc}
\toprule
 & \multicolumn{3}{c}{\textsc{Seed 01}} & \multicolumn{3}{c}{\textsc{Seed 02}} & \multicolumn{3}{c}{\textsc{Seed 03}} & \multicolumn{3}{c}{\textsc{Seed 04}} & \multicolumn{3}{c}{\textsc{Seed 05}} \\
\textsc{Method} & Train & Test & Bound & Train & Test & Bound & Train & Test & Bound & Train & Test & Bound & Train & Test & Bound \\
\cmidrule(lr){2-4}\cmidrule(lr){5-7}\cmidrule(lr){8-10}\cmidrule(lr){11-13}\cmidrule(lr){14-16}
NonRec-NonInf & 0.045 & 0.045 & 0.107 & 0.050 & 0.052 & 0.124 & 0.029 & 0.029 & 0.070 & 0.063 & 0.058 & 0.136 & 0.061 & 0.065 & 0.140 \\
NonRec-Inf & 0.033 & 0.033 & 0.046 & 0.043 & 0.046 & 0.063 & 0.018 & 0.018 & 0.027 & 0.041 & 0.036 & 0.050 & 0.045 & 0.049 & 0.061 \\
Rec T=2 & 0.040 & 0.040 & 0.055 & 0.043 & 0.046 & 0.060 & 0.018 & 0.018 & 0.026 & 0.044 & 0.038 & 0.051 & 0.041 & 0.045 & 0.054 \\
Rec T=6 & 0.037 & 0.037 & 0.048 & 0.045 & 0.046 & 0.059 & 0.021 & 0.021 & 0.030 & 0.042 & 0.042 & 0.055 & 0.046 & 0.047 & 0.059 \\
\bottomrule
\end{tabular}}
\vspace{0.5em}
\end{subtable}
\end{table}

\begin{table}
\centering
\caption{Normalized train error, test error, and bound for REDQ on MuJoCo's humanoid}
\begin{subtable}{\textwidth}
\centering
\caption{Starter policy}
\adjustbox{max width=0.9\textwidth}{
\begin{tabular}{lccccccccccccccc}
\toprule
 & \multicolumn{3}{c}{\textsc{Seed 01}} & \multicolumn{3}{c}{\textsc{Seed 02}} & \multicolumn{3}{c}{\textsc{Seed 03}} & \multicolumn{3}{c}{\textsc{Seed 04}} & \multicolumn{3}{c}{\textsc{Seed 05}} \\
\textsc{Method} & Train & Test & Bound & Train & Test & Bound & Train & Test & Bound & Train & Test & Bound & Train & Test & Bound \\
\cmidrule(lr){2-4}\cmidrule(lr){5-7}\cmidrule(lr){8-10}\cmidrule(lr){11-13}\cmidrule(lr){14-16}
NonRec-NonInf & 0.159 & 0.157 & 0.554 & 0.125 & 0.129 & 0.291 & 0.052 & 0.053 & 0.954 & 0.171 & 0.169 & 0.405 & 0.176 & 0.171 & 0.502 \\
NonRec-Inf & 0.143 & 0.142 & 0.543 & 0.108 & 0.108 & 0.279 & 0.044 & 0.043 & 0.530 & 0.160 & 0.161 & 0.393 & 0.167 & 0.161 & 0.382 \\
Rec T=2 & 0.147 & 0.144 & 0.365 & 0.125 & 0.126 & 0.292 & 0.051 & 0.053 & 0.428 & 0.166 & 0.166 & 0.332 & 0.172 & 0.170 & 0.350 \\
Rec T=6 & 0.115 & 0.112 & 0.289 & 0.092 & 0.090 & 0.216 & 0.018 & 0.019 & 0.191 & 0.117 & 0.119 & 0.255 & 0.142 & 0.139 & 0.306 \\
\bottomrule
\end{tabular}}
\vspace{0.5em}
\end{subtable}
\begin{subtable}{\textwidth}
\centering
\caption{Intermediate policy}
\adjustbox{max width=0.9\textwidth}{
\begin{tabular}{lccccccccccccccc}
\toprule
 & \multicolumn{3}{c}{\textsc{Seed 01}} & \multicolumn{3}{c}{\textsc{Seed 02}} & \multicolumn{3}{c}{\textsc{Seed 03}} & \multicolumn{3}{c}{\textsc{Seed 04}} & \multicolumn{3}{c}{\textsc{Seed 05}} \\
\textsc{Method} & Train & Test & Bound & Train & Test & Bound & Train & Test & Bound & Train & Test & Bound & Train & Test & Bound \\
\cmidrule(lr){2-4}\cmidrule(lr){5-7}\cmidrule(lr){8-10}\cmidrule(lr){11-13}\cmidrule(lr){14-16}
NonRec-NonInf & 0.151 & 0.157 & 0.418 & 0.103 & 0.105 & 0.241 & 0.066 & 0.066 & 0.130 & 0.110 & 0.112 & 0.322 & 0.116 & 0.122 & 0.343 \\
NonRec-Inf & 0.133 & 0.137 & 0.366 & 0.101 & 0.104 & 0.201 & 0.064 & 0.065 & 0.131 & 0.110 & 0.111 & 0.260 & 0.106 & 0.110 & 0.310 \\
Rec T=2 & 0.150 & 0.153 & 0.276 & 0.107 & 0.110 & 0.171 & 0.065 & 0.066 & 0.113 & 0.115 & 0.117 & 0.224 & 0.117 & 0.119 & 0.226 \\
Rec T=6 & 0.117 & 0.120 & 0.227 & 0.080 & 0.084 & 0.147 & 0.046 & 0.049 & 0.092 & 0.094 & 0.096 & 0.172 & 0.082 & 0.086 & 0.161 \\
\bottomrule
\end{tabular}}
\vspace{0.5em}
\end{subtable}
\begin{subtable}{\textwidth}
\centering
\caption{Expert policy}
\adjustbox{max width=0.9\textwidth}{
\begin{tabular}{lccccccccccccccc}
\toprule
 & \multicolumn{3}{c}{\textsc{Seed 01}} & \multicolumn{3}{c}{\textsc{Seed 02}} & \multicolumn{3}{c}{\textsc{Seed 03}} & \multicolumn{3}{c}{\textsc{Seed 04}} & \multicolumn{3}{c}{\textsc{Seed 05}} \\
\textsc{Method} & Train & Test & Bound & Train & Test & Bound & Train & Test & Bound & Train & Test & Bound & Train & Test & Bound \\
\cmidrule(lr){2-4}\cmidrule(lr){5-7}\cmidrule(lr){8-10}\cmidrule(lr){11-13}\cmidrule(lr){14-16}
NonRec-NonInf & 0.095 & 0.098 & 0.268 & 0.073 & 0.071 & 0.177 & 0.081 & 0.088 & 0.244 & 0.082 & 0.083 & 0.261 & 0.106 & 0.102 & 0.284 \\
NonRec-Inf & 0.104 & 0.107 & 0.229 & 0.067 & 0.066 & 0.153 & 0.062 & 0.066 & 0.188 & 0.085 & 0.086 & 0.217 & 0.101 & 0.098 & 0.241 \\
Rec T=2 & 0.109 & 0.114 & 0.193 & 0.072 & 0.071 & 0.121 & 0.074 & 0.077 & 0.153 & 0.115 & 0.115 & 0.223 & 0.135 & 0.128 & 0.238 \\
Rec T=6 & 0.102 & 0.107 & 0.176 & 0.067 & 0.067 & 0.110 & 0.076 & 0.078 & 0.147 & 0.091 & 0.092 & 0.175 & 0.101 & 0.099 & 0.190 \\
\bottomrule
\end{tabular}}
\vspace{0.5em}
\end{subtable}
\end{table}

\begin{table}
\centering
\caption{Normalized train error, test error, and bound for REDQ on MuJoCo's hopper}
\begin{subtable}{\textwidth}
\centering
\caption{Starter policy}
\adjustbox{max width=0.9\textwidth}{
\begin{tabular}{lccccccccccccccc}
\toprule
 & \multicolumn{3}{c}{\textsc{Seed 01}} & \multicolumn{3}{c}{\textsc{Seed 02}} & \multicolumn{3}{c}{\textsc{Seed 03}} & \multicolumn{3}{c}{\textsc{Seed 04}} & \multicolumn{3}{c}{\textsc{Seed 05}} \\
\textsc{Method} & Train & Test & Bound & Train & Test & Bound & Train & Test & Bound & Train & Test & Bound & Train & Test & Bound \\
\cmidrule(lr){2-4}\cmidrule(lr){5-7}\cmidrule(lr){8-10}\cmidrule(lr){11-13}\cmidrule(lr){14-16}
NonRec-NonInf & 0.567 & 0.569 & 0.759 & 0.341 & 0.320 & 0.868 & 0.007 & 0.007 & 0.269 & 0.201 & 0.187 & 1.011 & 0.122 & 0.122 & 0.218 \\
NonRec-Inf & 0.575 & 0.577 & 0.854 & 0.347 & 0.326 & 1.019 & 0.009 & 0.009 & 0.283 & 0.200 & 0.186 & 1.324 & 0.189 & 0.189 & 0.314 \\
Rec T=2 & 0.580 & 0.583 & 0.856 & 0.346 & 0.326 & 1.014 & 0.009 & 0.009 & 0.282 & 0.200 & 0.187 & 1.260 & 0.261 & 0.261 & 0.389 \\
Rec T=6 & 0.156 & 0.156 & 0.321 & 0.243 & 0.228 & 0.803 & 0.004 & 0.004 & 0.399 & 0.166 & 0.155 & 0.859 & 0.058 & 0.058 & 0.123 \\
\bottomrule
\end{tabular}}
\vspace{0.5em}
\end{subtable}
\begin{subtable}{\textwidth}
\centering
\caption{Intermediate policy}
\adjustbox{max width=0.9\textwidth}{
\begin{tabular}{lccccccccccccccc}
\toprule
 & \multicolumn{3}{c}{\textsc{Seed 01}} & \multicolumn{3}{c}{\textsc{Seed 02}} & \multicolumn{3}{c}{\textsc{Seed 03}} & \multicolumn{3}{c}{\textsc{Seed 04}} & \multicolumn{3}{c}{\textsc{Seed 05}} \\
\textsc{Method} & Train & Test & Bound & Train & Test & Bound & Train & Test & Bound & Train & Test & Bound & Train & Test & Bound \\
\cmidrule(lr){2-4}\cmidrule(lr){5-7}\cmidrule(lr){8-10}\cmidrule(lr){11-13}\cmidrule(lr){14-16}
NonRec-NonInf & 0.079 & 0.081 & 0.181 & 0.075 & 0.076 & 0.165 & 0.073 & 0.074 & 0.197 & 0.096 & 0.095 & 0.233 & 0.060 & 0.060 & 0.155 \\
NonRec-Inf & 0.069 & 0.070 & 0.109 & 0.059 & 0.058 & 0.087 & 0.070 & 0.070 & 0.117 & 0.115 & 0.116 & 0.260 & 0.062 & 0.062 & 0.096 \\
Rec T=2 & 0.073 & 0.074 & 0.111 & 0.062 & 0.062 & 0.093 & 0.069 & 0.069 & 0.110 & 0.234 & 0.233 & 0.348 & 0.061 & 0.061 & 0.094 \\
Rec T=6 & 0.064 & 0.063 & 0.094 & 0.054 & 0.053 & 0.081 & 0.061 & 0.061 & 0.096 & 0.071 & 0.071 & 0.115 & 0.059 & 0.059 & 0.092 \\
\bottomrule
\end{tabular}}
\vspace{0.5em}
\end{subtable}
\begin{subtable}{\textwidth}
\centering
\caption{Expert policy}
\adjustbox{max width=0.9\textwidth}{
\begin{tabular}{lccccccccccccccc}
\toprule
 & \multicolumn{3}{c}{\textsc{Seed 01}} & \multicolumn{3}{c}{\textsc{Seed 02}} & \multicolumn{3}{c}{\textsc{Seed 03}} & \multicolumn{3}{c}{\textsc{Seed 04}} & \multicolumn{3}{c}{\textsc{Seed 05}} \\
\textsc{Method} & Train & Test & Bound & Train & Test & Bound & Train & Test & Bound & Train & Test & Bound & Train & Test & Bound \\
\cmidrule(lr){2-4}\cmidrule(lr){5-7}\cmidrule(lr){8-10}\cmidrule(lr){11-13}\cmidrule(lr){14-16}
NonRec-NonInf & 0.072 & 0.071 & 0.184 & 0.121 & 0.121 & 0.241 & 0.053 & 0.053 & 0.100 & 0.095 & 0.088 & 0.207 & 0.094 & 0.092 & 0.210 \\
NonRec-Inf & 0.058 & 0.057 & 0.089 & 0.060 & 0.061 & 0.091 & 0.046 & 0.046 & 0.068 & 0.106 & 0.098 & 0.244 & 0.060 & 0.059 & 0.095 \\
Rec T=2 & 0.058 & 0.058 & 0.090 & 0.062 & 0.062 & 0.090 & 0.044 & 0.044 & 0.060 & 0.196 & 0.181 & 0.305 & 0.066 & 0.065 & 0.096 \\
Rec T=6 & 0.056 & 0.056 & 0.086 & 0.061 & 0.061 & 0.086 & 0.040 & 0.041 & 0.054 & 0.058 & 0.058 & 0.095 & 0.057 & 0.056 & 0.084 \\
\bottomrule
\end{tabular}}
\vspace{0.5em}
\end{subtable}
\end{table}

\begin{table}
\centering
\caption{Normalized train error, test error, and bound for SAC on MuJoCo's cheetah}
\begin{subtable}{\textwidth}
\centering
\caption{Starter policy}
\adjustbox{max width=0.9\textwidth}{
\begin{tabular}{lccccccccc}
\toprule
 & \multicolumn{3}{c}{\textsc{Seed 01}} & \multicolumn{3}{c}{\textsc{Seed 02}} & \multicolumn{3}{c}{\textsc{Seed 03}} \\
\textsc{Method} & Train & Test & Bound & Train & Test & Bound & Train & Test & Bound \\
\cmidrule(lr){2-4}\cmidrule(lr){5-7}\cmidrule(lr){8-10}
NonRec-NonInf & 0.066 & 0.066 & 0.137 & 0.062 & 0.062 & 0.148 & 0.086 & 0.086 & 0.183 \\
NonRec-Inf & 0.077 & 0.077 & 0.140 & 0.050 & 0.050 & 0.096 & 0.069 & 0.069 & 0.139 \\
Rec T=2 & 0.059 & 0.059 & 0.107 & 0.063 & 0.063 & 0.112 & 0.051 & 0.051 & 0.088 \\
Rec T=6 & 0.036 & 0.036 & 0.063 & 0.048 & 0.049 & 0.074 & 0.035 & 0.035 & 0.061 \\
\bottomrule
\end{tabular}}
\vspace{0.5em}
\end{subtable}
\begin{subtable}{\textwidth}
\centering
\caption{Intermediate policy}
\adjustbox{max width=0.9\textwidth}{
\begin{tabular}{lccccccccc}
\toprule
 & \multicolumn{3}{c}{\textsc{Seed 01}} & \multicolumn{3}{c}{\textsc{Seed 02}} & \multicolumn{3}{c}{\textsc{Seed 03}} \\
\textsc{Method} & Train & Test & Bound & Train & Test & Bound & Train & Test & Bound \\
\cmidrule(lr){2-4}\cmidrule(lr){5-7}\cmidrule(lr){8-10}
NonRec-NonInf & 0.042 & 0.042 & 0.080 & 0.064 & 0.065 & 0.155 & 0.038 & 0.040 & 0.143 \\
NonRec-Inf & 0.029 & 0.029 & 0.044 & 0.057 & 0.058 & 0.113 & 0.038 & 0.041 & 0.062 \\
Rec T=2 & 0.023 & 0.023 & 0.039 & 0.055 & 0.056 & 0.079 & 0.046 & 0.048 & 0.068 \\
Rec T=6 & 0.020 & 0.020 & 0.031 & 0.051 & 0.051 & 0.073 & 0.031 & 0.034 & 0.049 \\
\bottomrule
\end{tabular}}
\vspace{0.5em}
\end{subtable}
\begin{subtable}{\textwidth}
\centering
\caption{Expert policy}
\adjustbox{max width=0.9\textwidth}{
\begin{tabular}{lccccccccc}
\toprule
 & \multicolumn{3}{c}{\textsc{Seed 01}} & \multicolumn{3}{c}{\textsc{Seed 02}} & \multicolumn{3}{c}{\textsc{Seed 03}} \\
\textsc{Method} & Train & Test & Bound & Train & Test & Bound & Train & Test & Bound \\
\cmidrule(lr){2-4}\cmidrule(lr){5-7}\cmidrule(lr){8-10}
NonRec-NonInf & 0.042 & 0.042 & 0.097 & 0.401 & 0.400 & 0.404 & 0.305 & 0.298 & 0.327 \\
NonRec-Inf & 0.034 & 0.034 & 0.059 & 0.401 & 0.400 & 0.406 & 0.109 & 0.108 & 0.169 \\
Rec T=2 & 0.028 & 0.028 & 0.045 & 0.066 & 0.066 & 0.113 & 0.045 & 0.051 & 0.071 \\
Rec T=6 & 0.030 & 0.030 & 0.045 & 0.049 & 0.049 & 0.079 & 0.037 & 0.037 & 0.053 \\
\bottomrule
\end{tabular}}
\vspace{0.5em}
\end{subtable}
\end{table}

\begin{table}
\centering
\caption{Normalized train error, test error, and bound for SAC on MuJoCo's humanoid}
\begin{subtable}{\textwidth}
\centering
\caption{Starter policy}
\adjustbox{max width=0.9\textwidth}{
\begin{tabular}{lccccccccc}
\toprule
 & \multicolumn{3}{c}{\textsc{Seed 01}} & \multicolumn{3}{c}{\textsc{Seed 02}} & \multicolumn{3}{c}{\textsc{Seed 03}} \\
\textsc{Method} & Train & Test & Bound & Train & Test & Bound & Train & Test & Bound \\
\cmidrule(lr){2-4}\cmidrule(lr){5-7}\cmidrule(lr){8-10}
NonRec-NonInf & 0.248 & 0.219 & 3.198 & 0.240 & 0.243 & 1.620 & 0.194 & 0.198 & 2.533 \\
NonRec-Inf & 0.167 & 0.145 & 2.051 & 0.741 & 0.740 & 1.648 & 0.219 & 0.233 & 2.400 \\
Rec T=2 & 0.150 & 0.151 & 1.049 & 0.202 & 0.196 & 0.610 & 0.144 & 0.151 & 0.873 \\
Rec T=6 & 0.066 & 0.053 & 0.538 & 0.133 & 0.126 & 0.470 & 0.028 & 0.028 & 0.340 \\
\bottomrule
\end{tabular}}
\vspace{0.5em}
\end{subtable}
\begin{subtable}{\textwidth}
\centering
\caption{Intermediate policy}
\adjustbox{max width=0.9\textwidth}{
\begin{tabular}{lccccccccc}
\toprule
 & \multicolumn{3}{c}{\textsc{Seed 01}} & \multicolumn{3}{c}{\textsc{Seed 02}} & \multicolumn{3}{c}{\textsc{Seed 03}} \\
\textsc{Method} & Train & Test & Bound & Train & Test & Bound & Train & Test & Bound \\
\cmidrule(lr){2-4}\cmidrule(lr){5-7}\cmidrule(lr){8-10}
NonRec-NonInf & 0.087 & 0.089 & 0.332 & 0.095 & 0.094 & 0.435 & 0.093 & 0.094 & 0.372 \\
NonRec-Inf & 0.077 & 0.080 & 0.278 & 0.101 & 0.101 & 0.391 & 0.088 & 0.089 & 0.304 \\
Rec T=2 & 0.099 & 0.100 & 0.241 & 0.089 & 0.088 & 0.268 & 0.071 & 0.072 & 0.164 \\
Rec T=6 & 0.083 & 0.084 & 0.203 & 0.090 & 0.090 & 0.207 & 0.079 & 0.080 & 0.171 \\
\bottomrule
\end{tabular}}
\vspace{0.5em}
\end{subtable}
\begin{subtable}{\textwidth}
\centering
\caption{Expert policy}
\adjustbox{max width=0.9\textwidth}{
\begin{tabular}{lccccccccc}
\toprule
 & \multicolumn{3}{c}{\textsc{Seed 01}} & \multicolumn{3}{c}{\textsc{Seed 02}} & \multicolumn{3}{c}{\textsc{Seed 03}} \\
\textsc{Method} & Train & Test & Bound & Train & Test & Bound & Train & Test & Bound \\
\cmidrule(lr){2-4}\cmidrule(lr){5-7}\cmidrule(lr){8-10}
NonRec-NonInf & 0.360 & 0.360 & 0.363 & 0.152 & 0.151 & 0.371 & 0.353 & 0.353 & 0.357 \\
NonRec-Inf & 0.360 & 0.360 & 0.365 & 0.103 & 0.100 & 0.339 & 0.353 & 0.353 & 0.358 \\
Rec T=2 & 0.061 & 0.061 & 0.127 & 0.127 & 0.122 & 0.288 & 0.078 & 0.078 & 0.168 \\
Rec T=6 & 0.062 & 0.062 & 0.127 & 0.126 & 0.122 & 0.259 & 0.065 & 0.065 & 0.128 \\
\bottomrule
\end{tabular}}
\vspace{0.5em}
\end{subtable}
\end{table}

\begin{table}
\centering
\caption{Normalized train error, test error, and bound for SAC on MuJoCo's hopper}
\begin{subtable}{\textwidth}
\centering
\caption{Starter policy}
\adjustbox{max width=0.9\textwidth}{
\begin{tabular}{lccccccccc}
\toprule
 & \multicolumn{3}{c}{\textsc{Seed 01}} & \multicolumn{3}{c}{\textsc{Seed 02}} & \multicolumn{3}{c}{\textsc{Seed 03}} \\
\textsc{Method} & Train & Test & Bound & Train & Test & Bound & Train & Test & Bound \\
\cmidrule(lr){2-4}\cmidrule(lr){5-7}\cmidrule(lr){8-10}
NonRec-NonInf & 0.399 & 0.402 & 0.700 & 0.392 & 0.391 & 0.692 & 0.543 & 0.542 & 0.799 \\
NonRec-Inf & 0.404 & 0.407 & 0.832 & 0.398 & 0.398 & 0.880 & 0.540 & 0.540 & 0.915 \\
Rec T=2 & 0.379 & 0.380 & 0.700 & 0.374 & 0.373 & 0.687 & 0.456 & 0.457 & 0.747 \\
Rec T=6 & 0.064 & 0.064 & 0.244 & 0.041 & 0.041 & 0.208 & 0.078 & 0.079 & 0.220 \\
\bottomrule
\end{tabular}}
\vspace{0.5em}
\end{subtable}
\begin{subtable}{\textwidth}
\centering
\caption{Intermediate policy}
\adjustbox{max width=0.9\textwidth}{
\begin{tabular}{lccccccccc}
\toprule
 & \multicolumn{3}{c}{\textsc{Seed 01}} & \multicolumn{3}{c}{\textsc{Seed 02}} & \multicolumn{3}{c}{\textsc{Seed 03}} \\
\textsc{Method} & Train & Test & Bound & Train & Test & Bound & Train & Test & Bound \\
\cmidrule(lr){2-4}\cmidrule(lr){5-7}\cmidrule(lr){8-10}
NonRec-NonInf & 0.100 & 0.100 & 0.243 & 0.080 & 0.080 & 0.159 & 0.109 & 0.109 & 0.234 \\
NonRec-Inf & 0.065 & 0.066 & 0.142 & 0.055 & 0.054 & 0.100 & 0.052 & 0.051 & 0.097 \\
Rec T=2 & 0.072 & 0.072 & 0.117 & 0.076 & 0.077 & 0.105 & 0.080 & 0.080 & 0.121 \\
Rec T=6 & 0.062 & 0.063 & 0.097 & 0.048 & 0.047 & 0.065 & 0.053 & 0.053 & 0.084 \\
\bottomrule
\end{tabular}}
\vspace{0.5em}
\end{subtable}
\begin{subtable}{\textwidth}
\centering
\caption{Expert policy}
\adjustbox{max width=0.9\textwidth}{
\begin{tabular}{lccccccccc}
\toprule
 & \multicolumn{3}{c}{\textsc{Seed 01}} & \multicolumn{3}{c}{\textsc{Seed 02}} & \multicolumn{3}{c}{\textsc{Seed 03}} \\
\textsc{Method} & Train & Test & Bound & Train & Test & Bound & Train & Test & Bound \\
\cmidrule(lr){2-4}\cmidrule(lr){5-7}\cmidrule(lr){8-10}
NonRec-NonInf & 0.072 & 0.074 & 0.208 & 0.099 & 0.090 & 0.225 & 0.081 & 0.078 & 0.232 \\
NonRec-Inf & 0.067 & 0.068 & 0.104 & 0.060 & 0.057 & 0.115 & 0.068 & 0.066 & 0.138 \\
Rec T=2 & 0.072 & 0.073 & 0.110 & 0.057 & 0.055 & 0.100 & 0.073 & 0.072 & 0.113 \\
Rec T=6 & 0.061 & 0.063 & 0.088 & 0.063 & 0.062 & 0.099 & 0.080 & 0.080 & 0.122 \\
\bottomrule
\end{tabular}}
\vspace{0.5em}
\end{subtable}
\end{table}

\begin{table}
\centering
\caption{Normalized train error, test error, and bound for PPO on MuJoCo's Cheetah}
\begin{subtable}{\textwidth}
\centering
\caption{Starter policy}
\adjustbox{max width=0.9\textwidth}{
\begin{tabular}{lccccccccc}
\toprule
 & \multicolumn{3}{c}{\textsc{Seed 01}} & \multicolumn{3}{c}{\textsc{Seed 02}} & \multicolumn{3}{c}{\textsc{Seed 03}} \\
\textsc{Method} & Train & Test & Bound & Train & Test & Bound & Train & Test & Bound \\
\cmidrule(lr){2-4}\cmidrule(lr){5-7}\cmidrule(lr){8-10}
NonRec-NonInf & 0.166 & 0.166 & 0.277 & 0.074 & 0.074 & 0.185 & 0.087 & 0.086 & 0.198 \\
NonRec-Inf & 0.057 & 0.057 & 0.125 & 0.090 & 0.090 & 0.160 & 0.078 & 0.078 & 0.163 \\
Rec T=2 & 0.063 & 0.063 & 0.119 & 0.055 & 0.054 & 0.111 & 0.080 & 0.080 & 0.146 \\
Rec T=6 & 0.047 & 0.047 & 0.094 & 0.040 & 0.040 & 0.076 & 0.045 & 0.045 & 0.088 \\
\bottomrule
\end{tabular}}
\vspace{0.5em}
\end{subtable}
\begin{subtable}{\textwidth}
\centering
\caption{Intermediate policy}
\adjustbox{max width=0.9\textwidth}{
\begin{tabular}{lccccccccc}
\toprule
 & \multicolumn{3}{c}{\textsc{Seed 01}} & \multicolumn{3}{c}{\textsc{Seed 02}} & \multicolumn{3}{c}{\textsc{Seed 03}} \\
\textsc{Method} & Train & Test & Bound & Train & Test & Bound & Train & Test & Bound \\
\cmidrule(lr){2-4}\cmidrule(lr){5-7}\cmidrule(lr){8-10}
NonRec-NonInf & 0.054 & 0.053 & 0.137 & 0.048 & 0.048 & 0.103 & 0.040 & 0.040 & 0.086 \\
NonRec-Inf & 0.057 & 0.056 & 0.096 & 0.030 & 0.030 & 0.060 & 0.056 & 0.056 & 0.110 \\
Rec T=2 & 0.064 & 0.063 & 0.108 & 0.034 & 0.034 & 0.064 & 0.030 & 0.030 & 0.059 \\
Rec T=6 & 0.042 & 0.041 & 0.066 & 0.023 & 0.023 & 0.043 & 0.025 & 0.025 & 0.046 \\
\bottomrule
\end{tabular}}
\vspace{0.5em}
\end{subtable}
\begin{subtable}{\textwidth}
\centering
\caption{Expert policy}
\adjustbox{max width=0.9\textwidth}{
\begin{tabular}{lccccccccc}
\toprule
 & \multicolumn{3}{c}{\textsc{Seed 01}} & \multicolumn{3}{c}{\textsc{Seed 02}} & \multicolumn{3}{c}{\textsc{Seed 03}} \\
\textsc{Method} & Train & Test & Bound & Train & Test & Bound & Train & Test & Bound \\
\cmidrule(lr){2-4}\cmidrule(lr){5-7}\cmidrule(lr){8-10}
NonRec-NonInf & 0.065 & 0.065 & 0.167 & 0.063 & 0.065 & 0.112 & 0.041 & 0.041 & 0.099 \\
NonRec-Inf & 0.052 & 0.052 & 0.106 & 0.042 & 0.045 & 0.075 & 0.046 & 0.046 & 0.075 \\
Rec T=2 & 0.049 & 0.049 & 0.081 & 0.026 & 0.029 & 0.050 & 0.028 & 0.028 & 0.049 \\
Rec T=6 & 0.041 & 0.041 & 0.065 & 0.023 & 0.025 & 0.041 & 0.022 & 0.022 & 0.039 \\
\bottomrule
\end{tabular}}
\vspace{0.5em}
\end{subtable}
\end{table}

\begin{table}
\centering
\caption{Normalized train error, test error, and bound for PPO on MuJoCo's Humanoid}
\begin{subtable}{\textwidth}
\centering
\caption{Starter policy}
\adjustbox{max width=0.9\textwidth}{
\begin{tabular}{lccccccccc}
\toprule
 & \multicolumn{3}{c}{\textsc{Seed 01}} & \multicolumn{3}{c}{\textsc{Seed 02}} & \multicolumn{3}{c}{\textsc{Seed 03}} \\
\textsc{Method} & Train & Test & Bound & Train & Test & Bound & Train & Test & Bound \\
\cmidrule(lr){2-4}\cmidrule(lr){5-7}\cmidrule(lr){8-10}
NonRec-NonInf & 0.101 & 0.108 & 0.262 & 0.112 & 0.108 & 0.324 & 0.121 & 0.121 & 0.299 \\
NonRec-Inf & 0.099 & 0.104 & 0.266 & 0.105 & 0.102 & 0.336 & 0.104 & 0.105 & 0.259 \\
Rec T=2 & 0.097 & 0.103 & 0.212 & 0.099 & 0.097 & 0.249 & 0.101 & 0.100 & 0.208 \\
Rec T=6 & 0.053 & 0.058 & 0.147 & 0.055 & 0.050 & 0.146 & 0.052 & 0.050 & 0.137 \\
\bottomrule
\end{tabular}}
\vspace{0.5em}
\end{subtable}
\begin{subtable}{\textwidth}
\centering
\caption{Intermediate policy}
\adjustbox{max width=0.9\textwidth}{
\begin{tabular}{lccccccccc}
\toprule
 & \multicolumn{3}{c}{\textsc{Seed 01}} & \multicolumn{3}{c}{\textsc{Seed 02}} & \multicolumn{3}{c}{\textsc{Seed 03}} \\
\textsc{Method} & Train & Test & Bound & Train & Test & Bound & Train & Test & Bound \\
\cmidrule(lr){2-4}\cmidrule(lr){5-7}\cmidrule(lr){8-10}
NonRec-NonInf & 0.113 & 0.112 & 0.427 & 0.101 & 0.102 & 0.293 & 0.106 & 0.103 & 0.291 \\
NonRec-Inf & 0.104 & 0.103 & 0.358 & 0.095 & 0.097 & 0.256 & 0.105 & 0.101 & 0.274 \\
Rec T=2 & 0.106 & 0.106 & 0.247 & 0.098 & 0.100 & 0.222 & 0.108 & 0.103 & 0.239 \\
Rec T=6 & 0.045 & 0.045 & 0.122 & 0.042 & 0.043 & 0.115 & 0.065 & 0.061 & 0.170 \\
\bottomrule
\end{tabular}}
\vspace{0.5em}
\end{subtable}
\begin{subtable}{\textwidth}
\centering
\caption{Expert policy}
\adjustbox{max width=0.9\textwidth}{
\begin{tabular}{lccccccccc}
\toprule
 & \multicolumn{3}{c}{\textsc{Seed 01}} & \multicolumn{3}{c}{\textsc{Seed 02}} & \multicolumn{3}{c}{\textsc{Seed 03}} \\
\textsc{Method} & Train & Test & Bound & Train & Test & Bound & Train & Test & Bound \\
\cmidrule(lr){2-4}\cmidrule(lr){5-7}\cmidrule(lr){8-10}
NonRec-NonInf & 0.121 & 0.123 & 0.535 & 0.115 & 0.113 & 0.377 & 0.123 & 0.124 & 0.373 \\
NonRec-Inf & 0.086 & 0.084 & 0.385 & 0.100 & 0.100 & 0.346 & 0.085 & 0.085 & 0.267 \\
Rec T=2 & 0.101 & 0.102 & 0.237 & 0.134 & 0.131 & 0.297 & 0.101 & 0.103 & 0.213 \\
Rec T=6 & 0.050 & 0.049 & 0.153 & 0.051 & 0.053 & 0.123 & 0.045 & 0.045 & 0.108 \\
\bottomrule
\end{tabular}}
\vspace{0.5em}
\end{subtable}
\end{table}

\begin{table}
\centering
\caption{Normalized train error, test error, and bound for PPO on MuJoCo's Hopper}
\begin{subtable}{\textwidth}
\centering
\caption{Starter policy}
\adjustbox{max width=0.9\textwidth}{
\begin{tabular}{lccccccccc}
\toprule
 & \multicolumn{3}{c}{\textsc{Seed 01}} & \multicolumn{3}{c}{\textsc{Seed 02}} & \multicolumn{3}{c}{\textsc{Seed 03}} \\
\textsc{Method} & Train & Test & Bound & Train & Test & Bound & Train & Test & Bound \\
\cmidrule(lr){2-4}\cmidrule(lr){5-7}\cmidrule(lr){8-10}
NonRec-NonInf & 0.058 & 0.058 & 0.103 & 0.076 & 0.079 & 0.187 & 0.086 & 0.086 & 0.141 \\
NonRec-Inf & 0.059 & 0.059 & 0.130 & 0.071 & 0.068 & 0.124 & 0.086 & 0.086 & 0.168 \\
Rec T=2 & 0.059 & 0.060 & 0.127 & 0.066 & 0.065 & 0.116 & 0.088 & 0.088 & 0.170 \\
Rec T=6 & 0.011 & 0.011 & 0.057 & 0.062 & 0.067 & 0.104 & 0.013 & 0.012 & 0.057 \\
\bottomrule
\end{tabular}}
\vspace{0.5em}
\end{subtable}
\begin{subtable}{\textwidth}
\centering
\caption{Intermediate policy}
\adjustbox{max width=0.9\textwidth}{
\begin{tabular}{lccccccccc}
\toprule
 & \multicolumn{3}{c}{\textsc{Seed 01}} & \multicolumn{3}{c}{\textsc{Seed 02}} & \multicolumn{3}{c}{\textsc{Seed 03}} \\
\textsc{Method} & Train & Test & Bound & Train & Test & Bound & Train & Test & Bound \\
\cmidrule(lr){2-4}\cmidrule(lr){5-7}\cmidrule(lr){8-10}
NonRec-NonInf & 0.074 & 0.075 & 0.124 & 0.134 & 0.134 & 0.252 & 0.189 & 0.184 & 0.301 \\
NonRec-Inf & 0.076 & 0.076 & 0.146 & 0.067 & 0.064 & 0.121 & 0.078 & 0.080 & 0.148 \\
Rec T=2 & 0.077 & 0.077 & 0.151 & 0.076 & 0.074 & 0.131 & 0.075 & 0.074 & 0.125 \\
Rec T=6 & 0.016 & 0.016 & 0.065 & 0.061 & 0.059 & 0.108 & 0.066 & 0.067 & 0.112 \\
\bottomrule
\end{tabular}}
\vspace{0.5em}
\end{subtable}
\begin{subtable}{\textwidth}
\centering
\caption{Expert policy}
\adjustbox{max width=0.9\textwidth}{
\begin{tabular}{lccccccccc}
\toprule
 & \multicolumn{3}{c}{\textsc{Seed 01}} & \multicolumn{3}{c}{\textsc{Seed 02}} & \multicolumn{3}{c}{\textsc{Seed 03}} \\
\textsc{Method} & Train & Test & Bound & Train & Test & Bound & Train & Test & Bound \\
\cmidrule(lr){2-4}\cmidrule(lr){5-7}\cmidrule(lr){8-10}
NonRec-NonInf & 0.088 & 0.088 & 0.137 & 0.071 & 0.075 & 0.144 & 0.082 & 0.082 & 0.178 \\
NonRec-Inf & 0.089 & 0.089 & 0.158 & 0.073 & 0.075 & 0.135 & 0.059 & 0.060 & 0.101 \\
Rec T=2 & 0.091 & 0.091 & 0.163 & 0.082 & 0.086 & 0.150 & 0.062 & 0.062 & 0.102 \\
Rec T=6 & 0.024 & 0.024 & 0.072 & 0.056 & 0.059 & 0.095 & 0.052 & 0.053 & 0.086 \\
\bottomrule
\end{tabular}}
\vspace{0.5em}
\end{subtable}
\end{table}

\begin{table}
\centering
\caption{Normalized train error, test error, and bound for REDQ on Meta-world's Drawer open}
\begin{subtable}{\textwidth}
\centering
\caption{Starter policy}
\adjustbox{max width=0.9\textwidth}{
\begin{tabular}{lccccccccccccccc}
\toprule
 & \multicolumn{3}{c}{\textsc{Seed 01}} & \multicolumn{3}{c}{\textsc{Seed 02}} & \multicolumn{3}{c}{\textsc{Seed 03}} & \multicolumn{3}{c}{\textsc{Seed 04}} & \multicolumn{3}{c}{\textsc{Seed 05}} \\
\textsc{Method} & Train & Test & Bound & Train & Test & Bound & Train & Test & Bound & Train & Test & Bound & Train & Test & Bound \\
\cmidrule(lr){2-4}\cmidrule(lr){5-7}\cmidrule(lr){8-10}\cmidrule(lr){11-13}\cmidrule(lr){14-16}
NonRec-NonInf & 0.023 & 0.023 & 0.040 & 0.070 & 0.070 & 0.160 & 0.062 & 0.062 & 0.138 & 0.008 & 0.008 & 0.014 & 0.025 & 0.025 & 0.052 \\
NonRec-Inf & 0.013 & 0.013 & 0.020 & 0.069 & 0.069 & 0.088 & 0.067 & 0.068 & 0.086 & 0.007 & 0.007 & 0.013 & 0.021 & 0.021 & 0.031 \\
Rec T=2 & 0.014 & 0.014 & 0.022 & 0.063 & 0.063 & 0.082 & 0.055 & 0.055 & 0.071 & 0.009 & 0.009 & 0.015 & 0.023 & 0.023 & 0.033 \\
Rec T=6 & 0.019 & 0.019 & 0.030 & 0.060 & 0.060 & 0.075 & 0.052 & 0.052 & 0.066 & 0.021 & 0.021 & 0.029 & 0.019 & 0.019 & 0.028 \\
\bottomrule
\end{tabular}}
\vspace{0.5em}
\end{subtable}
\begin{subtable}{\textwidth}
\centering
\caption{Intermediate policy}
\adjustbox{max width=0.9\textwidth}{
\begin{tabular}{lccccccccccccccc}
\toprule
 & \multicolumn{3}{c}{\textsc{Seed 01}} & \multicolumn{3}{c}{\textsc{Seed 02}} & \multicolumn{3}{c}{\textsc{Seed 03}} & \multicolumn{3}{c}{\textsc{Seed 04}} & \multicolumn{3}{c}{\textsc{Seed 05}} \\
\textsc{Method} & Train & Test & Bound & Train & Test & Bound & Train & Test & Bound & Train & Test & Bound & Train & Test & Bound \\
\cmidrule(lr){2-4}\cmidrule(lr){5-7}\cmidrule(lr){8-10}\cmidrule(lr){11-13}\cmidrule(lr){14-16}
NonRec-NonInf & 0.068 & 0.067 & 0.128 & 0.157 & 0.158 & 0.243 & 0.005 & 0.005 & 0.010 & 0.005 & 0.005 & 0.009 & 0.050 & 0.050 & 0.089 \\
NonRec-Inf & 0.053 & 0.052 & 0.066 & 0.072 & 0.072 & 0.185 & 0.006 & 0.006 & 0.010 & 0.006 & 0.006 & 0.010 & 0.029 & 0.029 & 0.040 \\
Rec T=2 & 0.053 & 0.052 & 0.065 & 0.061 & 0.061 & 0.085 & 0.007 & 0.007 & 0.012 & 0.007 & 0.007 & 0.011 & 0.023 & 0.023 & 0.032 \\
Rec T=6 & 0.049 & 0.048 & 0.061 & 0.066 & 0.066 & 0.125 & 0.016 & 0.016 & 0.022 & 0.014 & 0.014 & 0.019 & 0.022 & 0.022 & 0.029 \\
\bottomrule
\end{tabular}}
\vspace{0.5em}
\end{subtable}
\begin{subtable}{\textwidth}
\centering
\caption{Expert policy}
\adjustbox{max width=0.9\textwidth}{
\begin{tabular}{lccccccccccccccc}
\toprule
 & \multicolumn{3}{c}{\textsc{Seed 01}} & \multicolumn{3}{c}{\textsc{Seed 02}} & \multicolumn{3}{c}{\textsc{Seed 03}} & \multicolumn{3}{c}{\textsc{Seed 04}} & \multicolumn{3}{c}{\textsc{Seed 05}} \\
\textsc{Method} & Train & Test & Bound & Train & Test & Bound & Train & Test & Bound & Train & Test & Bound & Train & Test & Bound \\
\cmidrule(lr){2-4}\cmidrule(lr){5-7}\cmidrule(lr){8-10}\cmidrule(lr){11-13}\cmidrule(lr){14-16}
NonRec-NonInf & 0.082 & 0.077 & 0.154 & 0.077 & 0.076 & 0.151 & 0.052 & 0.056 & 0.109 & 0.016 & 0.016 & 0.031 & 0.059 & 0.066 & 0.091 \\
NonRec-Inf & 0.069 & 0.064 & 0.083 & 0.054 & 0.054 & 0.068 & 0.042 & 0.046 & 0.058 & 0.010 & 0.010 & 0.017 & 0.045 & 0.048 & 0.065 \\
Rec T=2 & 0.067 & 0.063 & 0.077 & 0.059 & 0.059 & 0.074 & 0.041 & 0.046 & 0.052 & 0.011 & 0.011 & 0.018 & 0.047 & 0.051 & 0.067 \\
Rec T=6 & 0.068 & 0.064 & 0.078 & 0.050 & 0.050 & 0.062 & 0.042 & 0.047 & 0.055 & 0.012 & 0.012 & 0.019 & 0.044 & 0.046 & 0.060 \\
\bottomrule
\end{tabular}}
\vspace{0.5em}
\end{subtable}
\end{table}

\begin{table}
\centering
\caption{Normalized train error, test error, and bound for REDQ on Meta-world's Reach hard}
\begin{subtable}{\textwidth}
\centering
\caption{Starter policy}
\adjustbox{max width=0.9\textwidth}{
\begin{tabular}{lccccccccccccccc}
\toprule
 & \multicolumn{3}{c}{\textsc{Seed 01}} & \multicolumn{3}{c}{\textsc{Seed 02}} & \multicolumn{3}{c}{\textsc{Seed 03}} & \multicolumn{3}{c}{\textsc{Seed 04}} & \multicolumn{3}{c}{\textsc{Seed 05}} \\
\textsc{Method} & Train & Test & Bound & Train & Test & Bound & Train & Test & Bound & Train & Test & Bound & Train & Test & Bound \\
\cmidrule(lr){2-4}\cmidrule(lr){5-7}\cmidrule(lr){8-10}\cmidrule(lr){11-13}\cmidrule(lr){14-16}
NonRec-NonInf & 0.006 & 0.006 & 0.012 & 0.007 & 0.007 & 0.012 & 0.003 & 0.003 & 0.005 & 0.231 & 0.231 & 0.234 & 0.234 & 0.234 & 0.236 \\
NonRec-Inf & 0.005 & 0.005 & 0.008 & 0.005 & 0.005 & 0.007 & 0.002 & 0.002 & 0.004 & 0.231 & 0.231 & 0.235 & 0.234 & 0.234 & 0.237 \\
Rec T=2 & 0.005 & 0.005 & 0.008 & 0.005 & 0.005 & 0.007 & 0.002 & 0.002 & 0.004 & 0.135 & 0.135 & 0.163 & 0.227 & 0.227 & 0.232 \\
Rec T=6 & 0.005 & 0.005 & 0.007 & 0.004 & 0.004 & 0.006 & 0.002 & 0.002 & 0.003 & 0.223 & 0.223 & 0.230 & 0.231 & 0.231 & 0.235 \\
\bottomrule
\end{tabular}}
\vspace{0.5em}
\end{subtable}
\begin{subtable}{\textwidth}
\centering
\caption{Intermediate policy}
\adjustbox{max width=0.9\textwidth}{
\begin{tabular}{lccccccccccccccc}
\toprule
 & \multicolumn{3}{c}{\textsc{Seed 01}} & \multicolumn{3}{c}{\textsc{Seed 02}} & \multicolumn{3}{c}{\textsc{Seed 03}} & \multicolumn{3}{c}{\textsc{Seed 04}} & \multicolumn{3}{c}{\textsc{Seed 05}} \\
\textsc{Method} & Train & Test & Bound & Train & Test & Bound & Train & Test & Bound & Train & Test & Bound & Train & Test & Bound \\
\cmidrule(lr){2-4}\cmidrule(lr){5-7}\cmidrule(lr){8-10}\cmidrule(lr){11-13}\cmidrule(lr){14-16}
NonRec-NonInf & 0.193 & 0.193 & 0.195 & 0.005 & 0.005 & 0.009 & 0.199 & 0.199 & 0.201 & 0.183 & 0.183 & 0.186 & 0.186 & 0.186 & 0.188 \\
NonRec-Inf & 0.193 & 0.193 & 0.196 & 0.004 & 0.004 & 0.006 & 0.199 & 0.199 & 0.202 & 0.183 & 0.183 & 0.187 & 0.186 & 0.186 & 0.189 \\
Rec T=2 & 0.041 & 0.041 & 0.068 & 0.004 & 0.004 & 0.006 & 0.196 & 0.196 & 0.200 & 0.024 & 0.024 & 0.029 & 0.025 & 0.025 & 0.031 \\
Rec T=6 & 0.097 & 0.097 & 0.115 & 0.003 & 0.003 & 0.005 & 0.198 & 0.198 & 0.201 & 0.025 & 0.025 & 0.032 & 0.040 & 0.040 & 0.057 \\
\bottomrule
\end{tabular}}
\vspace{0.5em}
\end{subtable}
\begin{subtable}{\textwidth}
\centering
\caption{Expert policy}
\adjustbox{max width=0.9\textwidth}{
\begin{tabular}{lccccccccccccccc}
\toprule
 & \multicolumn{3}{c}{\textsc{Seed 01}} & \multicolumn{3}{c}{\textsc{Seed 02}} & \multicolumn{3}{c}{\textsc{Seed 03}} & \multicolumn{3}{c}{\textsc{Seed 04}} & \multicolumn{3}{c}{\textsc{Seed 05}} \\
\textsc{Method} & Train & Test & Bound & Train & Test & Bound & Train & Test & Bound & Train & Test & Bound & Train & Test & Bound \\
\cmidrule(lr){2-4}\cmidrule(lr){5-7}\cmidrule(lr){8-10}\cmidrule(lr){11-13}\cmidrule(lr){14-16}
NonRec-NonInf & 0.007 & 0.007 & 0.013 & 0.010 & 0.010 & 0.022 & 0.208 & 0.208 & 0.211 & 0.011 & 0.011 & 0.020 & 0.211 & 0.211 & 0.213 \\
NonRec-Inf & 0.006 & 0.006 & 0.009 & 0.012 & 0.012 & 0.016 & 0.208 & 0.208 & 0.211 & 0.009 & 0.009 & 0.012 & 0.211 & 0.211 & 0.214 \\
Rec T=2 & 0.005 & 0.005 & 0.007 & 0.011 & 0.011 & 0.014 & 0.203 & 0.203 & 0.208 & 0.008 & 0.008 & 0.011 & 0.207 & 0.207 & 0.211 \\
Rec T=6 & 0.005 & 0.005 & 0.007 & 0.010 & 0.010 & 0.014 & 0.172 & 0.172 & 0.182 & 0.007 & 0.007 & 0.010 & 0.209 & 0.209 & 0.213 \\
\bottomrule
\end{tabular}}
\vspace{0.5em}
\end{subtable}
\end{table}

\begin{table}
\centering
\caption{Normalized train error, test error, and bound for REDQ on Meta-world's Window open}
\begin{subtable}{\textwidth}
\centering
\caption{Starter policy}
\adjustbox{max width=0.9\textwidth}{
\begin{tabular}{lccccccccccccccc}
\toprule
 & \multicolumn{3}{c}{\textsc{Seed 01}} & \multicolumn{3}{c}{\textsc{Seed 02}} & \multicolumn{3}{c}{\textsc{Seed 03}} & \multicolumn{3}{c}{\textsc{Seed 04}} & \multicolumn{3}{c}{\textsc{Seed 05}} \\
\textsc{Method} & Train & Test & Bound & Train & Test & Bound & Train & Test & Bound & Train & Test & Bound & Train & Test & Bound \\
\cmidrule(lr){2-4}\cmidrule(lr){5-7}\cmidrule(lr){8-10}\cmidrule(lr){11-13}\cmidrule(lr){14-16}
NonRec-NonInf & 0.432 & 0.435 & 0.439 & 0.004 & 0.005 & 0.007 & 0.041 & 0.042 & 0.066 & 0.425 & 0.425 & 0.432 & 0.431 & 0.432 & 0.438 \\
NonRec-Inf & 0.432 & 0.435 & 0.441 & 0.005 & 0.005 & 0.008 & 0.032 & 0.032 & 0.042 & 0.425 & 0.426 & 0.435 & 0.431 & 0.432 & 0.442 \\
Rec T=2 & 0.077 & 0.076 & 0.134 & 0.005 & 0.005 & 0.008 & 0.034 & 0.034 & 0.045 & 0.056 & 0.056 & 0.068 & 0.059 & 0.059 & 0.094 \\
Rec T=6 & 0.073 & 0.071 & 0.148 & 0.010 & 0.011 & 0.015 & 0.033 & 0.032 & 0.042 & 0.054 & 0.054 & 0.082 & 0.059 & 0.059 & 0.107 \\
\bottomrule
\end{tabular}}
\vspace{0.5em}
\end{subtable}
\begin{subtable}{\textwidth}
\centering
\caption{Intermediate policy}
\adjustbox{max width=0.9\textwidth}{
\begin{tabular}{lccccccccccccccc}
\toprule
 & \multicolumn{3}{c}{\textsc{Seed 01}} & \multicolumn{3}{c}{\textsc{Seed 02}} & \multicolumn{3}{c}{\textsc{Seed 03}} & \multicolumn{3}{c}{\textsc{Seed 04}} & \multicolumn{3}{c}{\textsc{Seed 05}} \\
\textsc{Method} & Train & Test & Bound & Train & Test & Bound & Train & Test & Bound & Train & Test & Bound & Train & Test & Bound \\
\cmidrule(lr){2-4}\cmidrule(lr){5-7}\cmidrule(lr){8-10}\cmidrule(lr){11-13}\cmidrule(lr){14-16}
NonRec-NonInf & 0.066 & 0.066 & 0.137 & 0.209 & 0.209 & 0.262 & 0.440 & 0.440 & 0.452 & 0.001 & 0.001 & 0.002 & 0.066 & 0.064 & 0.128 \\
NonRec-Inf & 0.050 & 0.050 & 0.063 & 0.345 & 0.345 & 0.378 & 0.440 & 0.440 & 0.450 & 0.001 & 0.001 & 0.003 & 0.073 & 0.075 & 0.090 \\
Rec T=2 & 0.047 & 0.047 & 0.059 & 0.049 & 0.049 & 0.059 & 0.066 & 0.066 & 0.133 & 0.001 & 0.001 & 0.003 & 0.064 & 0.064 & 0.081 \\
Rec T=6 & 0.047 & 0.047 & 0.057 & 0.091 & 0.091 & 0.147 & 0.063 & 0.063 & 0.144 & 0.002 & 0.002 & 0.004 & 0.067 & 0.066 & 0.083 \\
\bottomrule
\end{tabular}}
\vspace{0.5em}
\end{subtable}
\begin{subtable}{\textwidth}
\centering
\caption{Expert policy}
\adjustbox{max width=0.9\textwidth}{
\begin{tabular}{lccccccccccccccc}
\toprule
 & \multicolumn{3}{c}{\textsc{Seed 01}} & \multicolumn{3}{c}{\textsc{Seed 02}} & \multicolumn{3}{c}{\textsc{Seed 03}} & \multicolumn{3}{c}{\textsc{Seed 04}} & \multicolumn{3}{c}{\textsc{Seed 05}} \\
\textsc{Method} & Train & Test & Bound & Train & Test & Bound & Train & Test & Bound & Train & Test & Bound & Train & Test & Bound \\
\cmidrule(lr){2-4}\cmidrule(lr){5-7}\cmidrule(lr){8-10}\cmidrule(lr){11-13}\cmidrule(lr){14-16}
NonRec-NonInf & 0.041 & 0.042 & 0.086 & 0.001 & 0.001 & 0.003 & 0.445 & 0.445 & 0.452 & 0.001 & 0.001 & 0.002 & 0.073 & 0.073 & 0.140 \\
NonRec-Inf & 0.032 & 0.033 & 0.042 & 0.002 & 0.002 & 0.003 & 0.445 & 0.445 & 0.454 & 0.002 & 0.002 & 0.004 & 0.054 & 0.054 & 0.070 \\
Rec T=2 & 0.033 & 0.034 & 0.043 & 0.002 & 0.002 & 0.004 & 0.064 & 0.064 & 0.118 & 0.002 & 0.002 & 0.004 & 0.047 & 0.047 & 0.057 \\
Rec T=6 & 0.029 & 0.030 & 0.036 & 0.003 & 0.003 & 0.005 & 0.060 & 0.059 & 0.142 & 0.002 & 0.002 & 0.005 & 0.051 & 0.051 & 0.064 \\
\bottomrule
\end{tabular}}
\vspace{0.5em}
\end{subtable}
\end{table}

\begin{table}
\centering
\caption{Normalized train error, test error, and bound for REDQ on DM Control's Walker}
\begin{subtable}{\textwidth}
\centering
\caption{Starter policy}
\adjustbox{max width=0.9\textwidth}{
\begin{tabular}{lccccccccccccccc}
\toprule
 & \multicolumn{3}{c}{\textsc{Seed 01}} & \multicolumn{3}{c}{\textsc{Seed 02}} & \multicolumn{3}{c}{\textsc{Seed 03}} & \multicolumn{3}{c}{\textsc{Seed 04}} & \multicolumn{3}{c}{\textsc{Seed 05}} \\
\textsc{Method} & Train & Test & Bound & Train & Test & Bound & Train & Test & Bound & Train & Test & Bound & Train & Test & Bound \\
\cmidrule(lr){2-4}\cmidrule(lr){5-7}\cmidrule(lr){8-10}\cmidrule(lr){11-13}\cmidrule(lr){14-16}
NonRec-NonInf & 0.094 & 0.094 & 0.156 & 0.090 & 0.089 & 0.147 & 0.066 & 0.066 & 0.112 & 0.056 & 0.056 & 0.097 & 0.102 & 0.102 & 0.160 \\
NonRec-Inf & 0.084 & 0.084 & 0.152 & 0.082 & 0.082 & 0.148 & 0.054 & 0.054 & 0.107 & 0.045 & 0.045 & 0.095 & 0.073 & 0.072 & 0.133 \\
Rec T=2 & 0.089 & 0.089 & 0.155 & 0.081 & 0.081 & 0.144 & 0.058 & 0.058 & 0.112 & 0.052 & 0.052 & 0.102 & 0.081 & 0.080 & 0.142 \\
Rec T=6 & 0.062 & 0.062 & 0.117 & 0.057 & 0.057 & 0.107 & 0.039 & 0.039 & 0.084 & 0.034 & 0.034 & 0.077 & 0.053 & 0.052 & 0.101 \\
\bottomrule
\end{tabular}}
\vspace{0.5em}
\end{subtable}
\begin{subtable}{\textwidth}
\centering
\caption{Intermediate policy}
\adjustbox{max width=0.9\textwidth}{
\begin{tabular}{lccccccccccccccc}
\toprule
 & \multicolumn{3}{c}{\textsc{Seed 01}} & \multicolumn{3}{c}{\textsc{Seed 02}} & \multicolumn{3}{c}{\textsc{Seed 03}} & \multicolumn{3}{c}{\textsc{Seed 04}} & \multicolumn{3}{c}{\textsc{Seed 05}} \\
\textsc{Method} & Train & Test & Bound & Train & Test & Bound & Train & Test & Bound & Train & Test & Bound & Train & Test & Bound \\
\cmidrule(lr){2-4}\cmidrule(lr){5-7}\cmidrule(lr){8-10}\cmidrule(lr){11-13}\cmidrule(lr){14-16}
NonRec-NonInf & 0.095 & 0.095 & 0.153 & 0.089 & 0.089 & 0.146 & 0.090 & 0.090 & 0.139 & 0.083 & 0.087 & 0.131 & 0.100 & 0.100 & 0.160 \\
NonRec-Inf & 0.083 & 0.083 & 0.142 & 0.090 & 0.090 & 0.149 & 0.065 & 0.065 & 0.118 & 0.063 & 0.066 & 0.112 & 0.074 & 0.074 & 0.136 \\
Rec T=2 & 0.082 & 0.081 & 0.140 & 0.080 & 0.080 & 0.137 & 0.064 & 0.064 & 0.115 & 0.070 & 0.074 & 0.122 & 0.089 & 0.088 & 0.152 \\
Rec T=6 & 0.062 & 0.061 & 0.108 & 0.058 & 0.058 & 0.105 & 0.046 & 0.046 & 0.087 & 0.045 & 0.049 & 0.085 & 0.056 & 0.057 & 0.105 \\
\bottomrule
\end{tabular}}
\vspace{0.5em}
\end{subtable}
\begin{subtable}{\textwidth}
\centering
\caption{Expert policy}
\adjustbox{max width=0.9\textwidth}{
\begin{tabular}{lccccccccccccccc}
\toprule
 & \multicolumn{3}{c}{\textsc{Seed 01}} & \multicolumn{3}{c}{\textsc{Seed 02}} & \multicolumn{3}{c}{\textsc{Seed 03}} & \multicolumn{3}{c}{\textsc{Seed 04}} & \multicolumn{3}{c}{\textsc{Seed 05}} \\
\textsc{Method} & Train & Test & Bound & Train & Test & Bound & Train & Test & Bound & Train & Test & Bound & Train & Test & Bound \\
\cmidrule(lr){2-4}\cmidrule(lr){5-7}\cmidrule(lr){8-10}\cmidrule(lr){11-13}\cmidrule(lr){14-16}
NonRec-NonInf & 0.105 & 0.100 & 0.164 & 0.115 & 0.111 & 0.177 & 0.086 & 0.084 & 0.133 & 0.077 & 0.077 & 0.119 & 0.078 & 0.078 & 0.125 \\
NonRec-Inf & 0.097 & 0.093 & 0.168 & 0.082 & 0.077 & 0.130 & 0.077 & 0.076 & 0.130 & 0.064 & 0.065 & 0.116 & 0.071 & 0.071 & 0.122 \\
Rec T=2 & 0.097 & 0.092 & 0.166 & 0.089 & 0.085 & 0.139 & 0.068 & 0.066 & 0.124 & 0.062 & 0.062 & 0.109 & 0.075 & 0.075 & 0.128 \\
Rec T=6 & 0.069 & 0.067 & 0.121 & 0.060 & 0.058 & 0.100 & 0.052 & 0.050 & 0.096 & 0.046 & 0.045 & 0.083 & 0.054 & 0.054 & 0.096 \\
\bottomrule
\end{tabular}}
\vspace{0.5em}
\end{subtable}
\end{table}

\begin{table}
\centering
\caption{Normalized train error, test error, and bound for REDQ on DM Control's Reacher}
\begin{subtable}{\textwidth}
\centering
\caption{Starter policy}
\adjustbox{max width=0.9\textwidth}{
\begin{tabular}{lccccccccccccccc}
\toprule
 & \multicolumn{3}{c}{\textsc{Seed 01}} & \multicolumn{3}{c}{\textsc{Seed 02}} & \multicolumn{3}{c}{\textsc{Seed 03}} & \multicolumn{3}{c}{\textsc{Seed 04}} & \multicolumn{3}{c}{\textsc{Seed 05}} \\
\textsc{Method} & Train & Test & Bound & Train & Test & Bound & Train & Test & Bound & Train & Test & Bound & Train & Test & Bound \\
\cmidrule(lr){2-4}\cmidrule(lr){5-7}\cmidrule(lr){8-10}\cmidrule(lr){11-13}\cmidrule(lr){14-16}
NonRec-NonInf & 0.213 & 0.248 & 0.285 & 0.208 & 0.228 & 0.278 & 0.186 & 0.195 & 0.229 & 0.174 & 0.132 & 0.212 & 0.185 & 0.198 & 0.254 \\
NonRec-Inf & 0.201 & 0.243 & 0.282 & 0.199 & 0.216 & 0.273 & 0.188 & 0.194 & 0.273 & 0.172 & 0.134 & 0.263 & 0.165 & 0.183 & 0.242 \\
Rec T=2 & 0.209 & 0.230 & 0.283 & 0.189 & 0.212 & 0.272 & 0.182 & 0.189 & 0.264 & 0.173 & 0.148 & 0.254 & 0.166 & 0.184 & 0.234 \\
Rec T=6 & 0.183 & 0.218 & 0.257 & 0.176 & 0.200 & 0.248 & 0.178 & 0.183 & 0.254 & 0.167 & 0.138 & 0.246 & 0.149 & 0.209 & 0.210 \\
\bottomrule
\end{tabular}}
\vspace{0.5em}
\end{subtable}
\begin{subtable}{\textwidth}
\centering
\caption{Intermediate policy}
\adjustbox{max width=0.9\textwidth}{
\begin{tabular}{lccccccccccccccc}
\toprule
 & \multicolumn{3}{c}{\textsc{Seed 01}} & \multicolumn{3}{c}{\textsc{Seed 02}} & \multicolumn{3}{c}{\textsc{Seed 03}} & \multicolumn{3}{c}{\textsc{Seed 04}} & \multicolumn{3}{c}{\textsc{Seed 05}} \\
\textsc{Method} & Train & Test & Bound & Train & Test & Bound & Train & Test & Bound & Train & Test & Bound & Train & Test & Bound \\
\cmidrule(lr){2-4}\cmidrule(lr){5-7}\cmidrule(lr){8-10}\cmidrule(lr){11-13}\cmidrule(lr){14-16}
NonRec-NonInf & 0.143 & 0.129 & 0.213 & 0.103 & 0.125 & 0.162 & 0.168 & 0.198 & 0.243 & 0.262 & 0.264 & 0.327 & 0.109 & 0.103 & 0.169 \\
NonRec-Inf & 0.120 & 0.091 & 0.179 & 0.091 & 0.121 & 0.147 & 0.147 & 0.177 & 0.201 & 0.234 & 0.234 & 0.299 & 0.090 & 0.084 & 0.157 \\
Rec T=2 & 0.117 & 0.084 & 0.172 & 0.092 & 0.118 & 0.147 & 0.147 & 0.176 & 0.205 & 0.235 & 0.235 & 0.301 & 0.091 & 0.084 & 0.159 \\
Rec T=6 & 0.107 & 0.067 & 0.158 & 0.072 & 0.099 & 0.120 & 0.124 & 0.147 & 0.175 & 0.225 & 0.224 & 0.287 & 0.073 & 0.067 & 0.140 \\
\bottomrule
\end{tabular}}
\vspace{0.5em}
\end{subtable}
\begin{subtable}{\textwidth}
\centering
\caption{Expert policy}
\adjustbox{max width=0.9\textwidth}{
\begin{tabular}{lccccccccccccccc}
\toprule
 & \multicolumn{3}{c}{\textsc{Seed 01}} & \multicolumn{3}{c}{\textsc{Seed 02}} & \multicolumn{3}{c}{\textsc{Seed 03}} & \multicolumn{3}{c}{\textsc{Seed 04}} & \multicolumn{3}{c}{\textsc{Seed 05}} \\
\textsc{Method} & Train & Test & Bound & Train & Test & Bound & Train & Test & Bound & Train & Test & Bound & Train & Test & Bound \\
\cmidrule(lr){2-4}\cmidrule(lr){5-7}\cmidrule(lr){8-10}\cmidrule(lr){11-13}\cmidrule(lr){14-16}
NonRec-NonInf & 0.118 & 0.121 & 0.184 & 0.099 & 0.122 & 0.153 & 0.202 & 0.180 & 0.274 & 0.112 & 0.143 & 0.171 & 0.098 & 0.108 & 0.153 \\
NonRec-Inf & 0.102 & 0.107 & 0.177 & 0.064 & 0.099 & 0.105 & 0.196 & 0.174 & 0.254 & 0.105 & 0.148 & 0.149 & 0.097 & 0.113 & 0.162 \\
Rec T=2 & 0.104 & 0.109 & 0.172 & 0.073 & 0.107 & 0.114 & 0.197 & 0.174 & 0.251 & 0.104 & 0.142 & 0.145 & 0.090 & 0.102 & 0.143 \\
Rec T=6 & 0.114 & 0.119 & 0.229 & 0.052 & 0.117 & 0.087 & 0.182 & 0.151 & 0.236 & 0.084 & 0.123 & 0.118 & 0.072 & 0.084 & 0.119 \\
\bottomrule
\end{tabular}}
\vspace{0.5em}
\end{subtable}
\end{table}

\begin{table}
\centering
\caption{Normalized train error, test error, and bound for REDQ on DM Control's Ball-in-cup}
\begin{subtable}{\textwidth}
\centering
\caption{Starter policy}
\adjustbox{max width=0.9\textwidth}{
\begin{tabular}{lccccccccccccccc}
\toprule
 & \multicolumn{3}{c}{\textsc{Seed 01}} & \multicolumn{3}{c}{\textsc{Seed 02}} & \multicolumn{3}{c}{\textsc{Seed 03}} & \multicolumn{3}{c}{\textsc{Seed 04}} & \multicolumn{3}{c}{\textsc{Seed 05}} \\
\textsc{Method} & Train & Test & Bound & Train & Test & Bound & Train & Test & Bound & Train & Test & Bound & Train & Test & Bound \\
\cmidrule(lr){2-4}\cmidrule(lr){5-7}\cmidrule(lr){8-10}\cmidrule(lr){11-13}\cmidrule(lr){14-16}
NonRec-NonInf & 0.082 & 0.082 & 0.141 & 0.090 & 0.090 & 0.146 & 0.087 & 0.087 & 0.145 & 0.081 & 0.081 & 0.134 & 0.075 & 0.075 & 0.127 \\
NonRec-Inf & 0.065 & 0.066 & 0.104 & 0.072 & 0.072 & 0.120 & 0.077 & 0.077 & 0.121 & 0.066 & 0.066 & 0.113 & 0.066 & 0.066 & 0.107 \\
Rec T=2 & 0.068 & 0.068 & 0.109 & 0.068 & 0.068 & 0.106 & 0.065 & 0.065 & 0.106 & 0.068 & 0.068 & 0.109 & 0.070 & 0.070 & 0.109 \\
Rec T=6 & 0.055 & 0.055 & 0.090 & 0.047 & 0.047 & 0.079 & 0.050 & 0.050 & 0.085 & 0.051 & 0.051 & 0.085 & 0.047 & 0.048 & 0.080 \\
\bottomrule
\end{tabular}}
\vspace{0.5em}
\end{subtable}
\begin{subtable}{\textwidth}
\centering
\caption{Intermediate policy}
\adjustbox{max width=0.9\textwidth}{
\begin{tabular}{lccccccccccccccc}
\toprule
 & \multicolumn{3}{c}{\textsc{Seed 01}} & \multicolumn{3}{c}{\textsc{Seed 02}} & \multicolumn{3}{c}{\textsc{Seed 03}} & \multicolumn{3}{c}{\textsc{Seed 04}} & \multicolumn{3}{c}{\textsc{Seed 05}} \\
\textsc{Method} & Train & Test & Bound & Train & Test & Bound & Train & Test & Bound & Train & Test & Bound & Train & Test & Bound \\
\cmidrule(lr){2-4}\cmidrule(lr){5-7}\cmidrule(lr){8-10}\cmidrule(lr){11-13}\cmidrule(lr){14-16}
NonRec-NonInf & 0.101 & 0.101 & 0.160 & 0.079 & 0.079 & 0.135 & 0.066 & 0.066 & 0.119 & 0.094 & 0.093 & 0.149 & 0.077 & 0.076 & 0.131 \\
NonRec-Inf & 0.088 & 0.087 & 0.140 & 0.068 & 0.068 & 0.108 & 0.068 & 0.068 & 0.108 & 0.076 & 0.075 & 0.119 & 0.072 & 0.071 & 0.112 \\
Rec T=2 & 0.081 & 0.079 & 0.121 & 0.068 & 0.067 & 0.106 & 0.070 & 0.070 & 0.115 & 0.082 & 0.081 & 0.128 & 0.092 & 0.090 & 0.140 \\
Rec T=6 & 0.065 & 0.064 & 0.101 & 0.049 & 0.049 & 0.080 & 0.051 & 0.051 & 0.085 & 0.052 & 0.052 & 0.085 & 0.055 & 0.055 & 0.088 \\
\bottomrule
\end{tabular}}
\vspace{0.5em}
\end{subtable}
\begin{subtable}{\textwidth}
\centering
\caption{Expert policy}
\adjustbox{max width=0.9\textwidth}{
\begin{tabular}{lccccccccccccccc}
\toprule
 & \multicolumn{3}{c}{\textsc{Seed 01}} & \multicolumn{3}{c}{\textsc{Seed 02}} & \multicolumn{3}{c}{\textsc{Seed 03}} & \multicolumn{3}{c}{\textsc{Seed 04}} & \multicolumn{3}{c}{\textsc{Seed 05}} \\
\textsc{Method} & Train & Test & Bound & Train & Test & Bound & Train & Test & Bound & Train & Test & Bound & Train & Test & Bound \\
\cmidrule(lr){2-4}\cmidrule(lr){5-7}\cmidrule(lr){8-10}\cmidrule(lr){11-13}\cmidrule(lr){14-16}
NonRec-NonInf & 0.092 & 0.092 & 0.140 & 0.074 & 0.074 & 0.125 & 0.075 & 0.076 & 0.126 & 0.070 & 0.070 & 0.124 & 0.093 & 0.093 & 0.142 \\
NonRec-Inf & 0.073 & 0.073 & 0.116 & 0.073 & 0.073 & 0.118 & 0.071 & 0.071 & 0.117 & 0.071 & 0.071 & 0.116 & 0.067 & 0.067 & 0.109 \\
Rec T=2 & 0.068 & 0.068 & 0.110 & 0.074 & 0.074 & 0.118 & 0.067 & 0.067 & 0.106 & 0.071 & 0.071 & 0.110 & 0.077 & 0.077 & 0.120 \\
Rec T=6 & 0.052 & 0.052 & 0.086 & 0.054 & 0.054 & 0.090 & 0.052 & 0.052 & 0.086 & 0.053 & 0.053 & 0.089 & 0.050 & 0.049 & 0.082 \\
\bottomrule
\end{tabular}}
\vspace{0.5em}
\end{subtable}
\end{table}